\documentclass{article}

\usepackage[table,svgnames]{xcolor}

\usepackage{microtype}
\usepackage{graphicx}
\usepackage{subcaption}
\usepackage{multirow}
\usepackage{booktabs} % for professional tables
\usepackage{wrapfig}
\usepackage{listings}
\definecolor{dl_green}{rgb}{0.0, 0.6, 0.0} % Comments
\definecolor{dl_purple}{rgb}{0.5, 0.0, 0.5} % Keywords/Functions
\definecolor{dl_blue}{rgb}{0.0, 0.0, 1.0}   % Built-ins
\definecolor{dl_gray}{rgb}{0.97, 0.97, 0.97} % Background
\definecolor{dl_darkred}{rgb}{0.6, 0.0, 0.0} % Operators
\definecolor{myyellow}{HTML}{FFAE42}
\usepackage{enumitem}

\lstdefinelanguage{ModernPython}{
    language=Python,
    basicstyle=\fontsize{7}{9}\ttfamily, % slightly smaller, tighter
    backgroundcolor=\color{dl_gray},
    commentstyle=\color{gray},
    keywordstyle=\color{dl_purple}\bfseries,
    morekeywords={with, no_grad},
    classoffset=1,
    morekeywords={randn, randn_like, sqrt, zeros,size, expand, mse_loss, mean, squeeze, rand, sample, backward, step, zero_grad, push, optimizer},
    keywordstyle=\color{dl_blue},
    classoffset=0,
    classoffset=2,
    morekeywords={model_fn, solver_fn, score_fn},
    keywordstyle=\color{dl_darkred}\bfseries,
    classoffset=0,
    framesep=8pt,
    xleftmargin=8pt,
    frame=none,
    showstringspaces=false,
}

\usepackage[framemethod=TikZ]{mdframed}
\definecolor{myyellowlight}{HTML}{FBF9F5}

\definecolor{lightgray}{HTML}{F2F2F2}
\newmdenv[
    backgroundcolor=myyellowlight,
    roundcorner=3pt,
    skipabove=7pt,
    linewidth=0pt,
    innertopmargin=6pt,
    innerbottommargin=6pt, % Added for symmetry
    innerleftmargin=6pt,   % <--- This fixes the large left inset
    innerrightmargin=6pt   % <--- This fixes the large right inset
]{myframe}
\usepackage{hyperref}

\usepackage{amsmath,amsfonts,bm,mathtools,amssymb}

\def\1{\bm{1}}

\def\gA{{\mathcal{A}}}
\def\gB{{\mathcal{B}}}

\def\gL{{\mathcal{L}}}
\def\gM{{\mathcal{M}}}
\def\gN{{\mathcal{N}}}

\def\sR{{\mathbb{R}}}
\def\sS{{\mathbb{S}}}

\DeclareMathOperator{\E}{\mathbb{E}}

\newcommand{\R}{\mathbb{R}}

\newcommand{\norm}[1]{\left\lVert#1\right\rVert}

\newcommand*\diff{\mathop{}\!\mathrm{d}}

\DeclarePairedDelimiterX{\infdivx}[2]{(}{)}{%
  #1\;\delimsize|\delimsize|\;#2%
}

\newcommand{\parr}[1]{\left (#1\right )}
\newcommand{\brac}[1]{\left [#1\right ]}
\newcommand{\set}[1]{\left\{#1\right\}}
\newcommand{\inner}[2]{\left\langle #1, #2 \right\rangle}
\newcommand{\eg}{{e.g.}}
\newcommand{\ie}{{i.e.}}

\usepackage[most]{tcolorbox}
\newcommand{\yellowboxuncenter}[1]{%
\begin{tcolorbox}[
    colback=myyellow!12,
    colframe=orange!50!black,
    boxrule=0.6pt,
    arc=1pt,
    left=1pt,
    right=1pt,
    top=1pt,
    bottom=1pt
]
#1
\end{tcolorbox}
}

\usepackage[preprint]{icml2026}

\definecolor{RoyalBlue}{rgb}{0.25, 0.41, 0.88}
\newcommand{\cellhi}{\cellcolor{myyellow!15}}

\newcommand{\cellhii}{\cellcolor{myyellow!35}}
\def\gM{{\mathcal{M}}}
\newcommand{\too}{\rightarrow}
\definecolor{cadetblue}{rgb}{0.37, 0.62, 0.63}
\definecolor{steelblue}{RGB}{70,130,180}

\usepackage{amsmath}
\usepackage{amssymb}
\usepackage{mathtools}
\usepackage{amsthm}
\usepackage{thmtools}
\usepackage{thm-restate}

\usepackage[capitalize,noabbrev]{cleveref}

\theoremstyle{plain}
\newtheorem{theorem}{Theorem}[section]

\theoremstyle{definition}

\theoremstyle{remark}

\usepackage[textsize=tiny]{todonotes}

\icmltitlerunning{Flow Sampling}

\begin{document}

\twocolumn[
  \icmltitle{Flow Sampling: Learning to Sample from Unnormalized Densities\\ via Denoising Conditional Processes}

  \icmlsetsymbol{equal}{*}

  \begin{icmlauthorlist}
    \icmlauthor{Aaron Havens}{meta}
    \icmlauthor{Brian Karrer}{meta}
    \icmlauthor{Neta Shaul}{weizmann}
  \end{icmlauthorlist}

  \icmlaffiliation{meta}{FAIR at Meta}
  \icmlaffiliation{weizmann}{Weizmann Institute of Science}

  \icmlcorrespondingauthor{Aaron Havens}{havensaaronj@gmail.com}
  \icmlcorrespondingauthor{Neta Shaul}{shaulneta@gmail.com}

  %\icmlkeywords{Machine Learning, ICML}

  \vskip 0.3in
]

\printAffiliationsAndNotice{}  % no special notice (required even if empty)

\begin{abstract}

Sampling from unnormalized densities is analogous to the generative modeling problem, but the target distribution is defined by a known energy function instead of data samples.
Because evaluating the energy function is often costly, a primary challenge is to learn an efficient sampler.
We introduce \emph{Flow Sampling}, a framework built on diffusion models and flow matching for the data-free setting. Our training objective is conditioned on a noise sample and regresses onto a \emph{denoising} diffusion drift constructed from the energy function. In contrast, diffusion models' objective is conditioned on a data sample and regresses onto a \emph{noising} diffusion drift. We utilize the interpolant process to minimize the number of energy function evaluations during training, resulting in an efficient and scalable method for sampling unnormalized densities.
Furthermore, our formulation naturally extends to Riemannian manifolds, enabling diffusion-based sampling in geometries beyond Euclidean space. We derive a closed-form formula for the conditional drift on constant curvature manifolds, including hyperspheres and hyperbolic spaces.
We evaluate Flow Sampling on synthetic energy benchmarks, small peptides, large-scale amortized molecular conformer generation, and distributions supported on the sphere, demonstrating strong empirical performance.
\vspace{-10pt}
\end{abstract}
\section{Introduction}
\begin{figure*}[htbp]
    \begin{minipage}{\linewidth}
    \centering
    \includegraphics[width=0.85\linewidth]{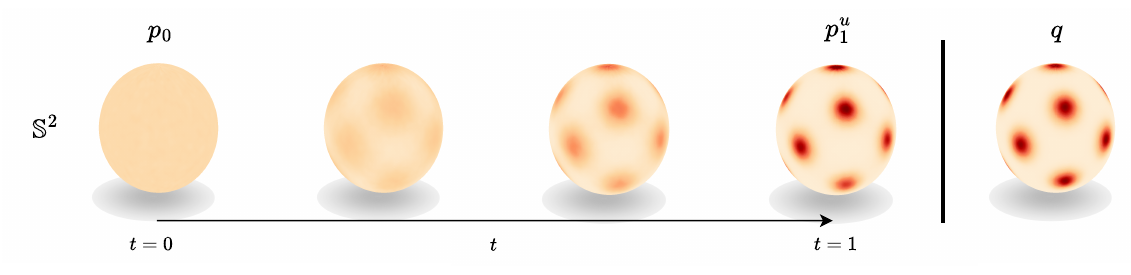}
    \end{minipage}
    % \vspace{-10pt}
    \caption{Flow Sampling is the first diffusion sampler able to perform sampling from an unnormalized density on Riemannian manifolds. The figure depicts a learned diffusion process using iterative Flow Sampling from a mixture of spherical von Mises--Fisher distributions.}
    \label{fig:vMF_vis}
    \vspace{-10pt}
\end{figure*}

Many problems in the computational sciences require sampling from high-dimensional probability distributions that are specified only up to a normalization constant~\citep{dfrenkel96:mc, noe2019boltzmann, barroso_omat24},
\begin{equation}\label{eq:target_distribution}
    q(x) =  \frac{\exp(r(x))}{Z},\quad Z = \int_{\R^d}\exp(r(x)) \diff{x} < \infty.
\end{equation}
The target distribution is $q$, and it is defined by the unnormalized log density $r(x)$, which is referred to as the \emph{reward function} in reinforcement learning, or the \emph{negative energy} in physical the sciences, computational chemistry and biology. One can typically evaluate $r(x)$ and its gradient $\nabla r(x)$, often at high computational cost, but does not have access to samples from $q$.

Markov chain Monte Carlo (MCMC) methods~\citep{hastings1970monte, neal2001annealed}, including Langevin dynamics~\citep{roberts1996exponential, roberts1998optimal}, converge asymptotically to the target distribution but often mix slowly and generate samples sequentially.
This limits their applicability in settings such as materials design~\cite{barroso_omat24} that require fast, reusable samplers across many target instances. This motivates the development of \emph{amortized} sampling methods that replace long sequential simulations with learned sampling dynamics.

% unified the to paragraph to a single one
In this work, we introduce \emph{Flow Sampling}, a diffusion-based framework for learning amortized samplers directly from unnormalized density functions.
Flow Sampling learns to sample the target distribution using a conditional denoising process constructed from the energy function, such that its marginal distribution matches the target by design.
It utilizes the known interpolant process and a detached state of the model to obtain reusable gradients of the negative energy which reduces considerably the required number of energy function evaluation.
In contrast, existing diffusion-based methods~\citep{phillips2024particle, de2024target, akhound2024iterated} learn to sample the target distribution through Monte Carlo corrections, such as importance weighting, resampling, or auxiliary MCMC. While principled, these approaches rely on stochastic estimation and repeated energy function evaluations, which can significantly increase computational cost.

More recently, diffusion samplers based on stochastic optimal control or Schr\"odinger bridge problems haved emerged. They learn the diffusion dynamics by optimizing divergences over diffusion path measures~\citep{zhang2022pis, AS, liu2025asbs}. Although theoretically well grounded, this perspective requires characterizing optimal controls or bridges, leading to complex training schemes and auxiliary networks. By building on the diffusion models~\cite{ho2020denoising, song2021score} and flow matching~\cite{lipman2023flow,liu2023flow, albergo2023stochastic} formulation, we alleviate these restrictions, resulting in a flexible design space and an easy to implement method. Despite the non-optimal control or coupling, in practice Flow Sampling cuts training cost by at least half on the large-scale amortized molecular conformer generation benchmarks~\citep{AS}. Moreover, the Flow Sampling naturally extends to Riemannian manifolds. Similar to Riemannian flow matching~\citep{chen2023riemannian}, we replace the affine interpolant with a geodesic interpolant, and furthermore we derive a closed-form formula for the conditional drift on constant curvature manifolds, which includes the hypersphere and hyperbolic spaces.

\textbf{Our main contributions:}
\begin{enumerate} %[itemsep=5pt]
    \item We introduce Flow Sampling, a principled framework for learning
    diffusion samplers from unnormalized densities by matching conditional
    denoising diffusion drifts.

    \item We derive a simple closed-form regression target that reuses
    endpoint energy gradients along the interpolant process, leading to an
    efficient replay-buffer training algorithm summarized in
    Algorithm~\ref{alg:fs_train}.

    \item We extend diffusion-based samplers to geometries beyond Euclidean spaces. In particular, we derive a closed-form formula for the conditional denoising diffusion drift on constant curvature Riemannian manifolds, including the hypersphere and the hyperbolic spaces.

    \item We demonstrate state-of-the-art performance on synthetic energy
    benchmarks and amortized molecular conformer generation, while reducing
    simulation training cost by $4$--$8$ times.
\end{enumerate}

\begin{figure*}
    \centering
    \begin{tabular}{ccc}
            Flow Matching conditional &  Flow Sampling conditional (ours) &  Marginal \\
            \includegraphics[
                trim=0 50 0 40,  % left, bottom, right, top
                clip,
                width=0.31\textwidth
            ]{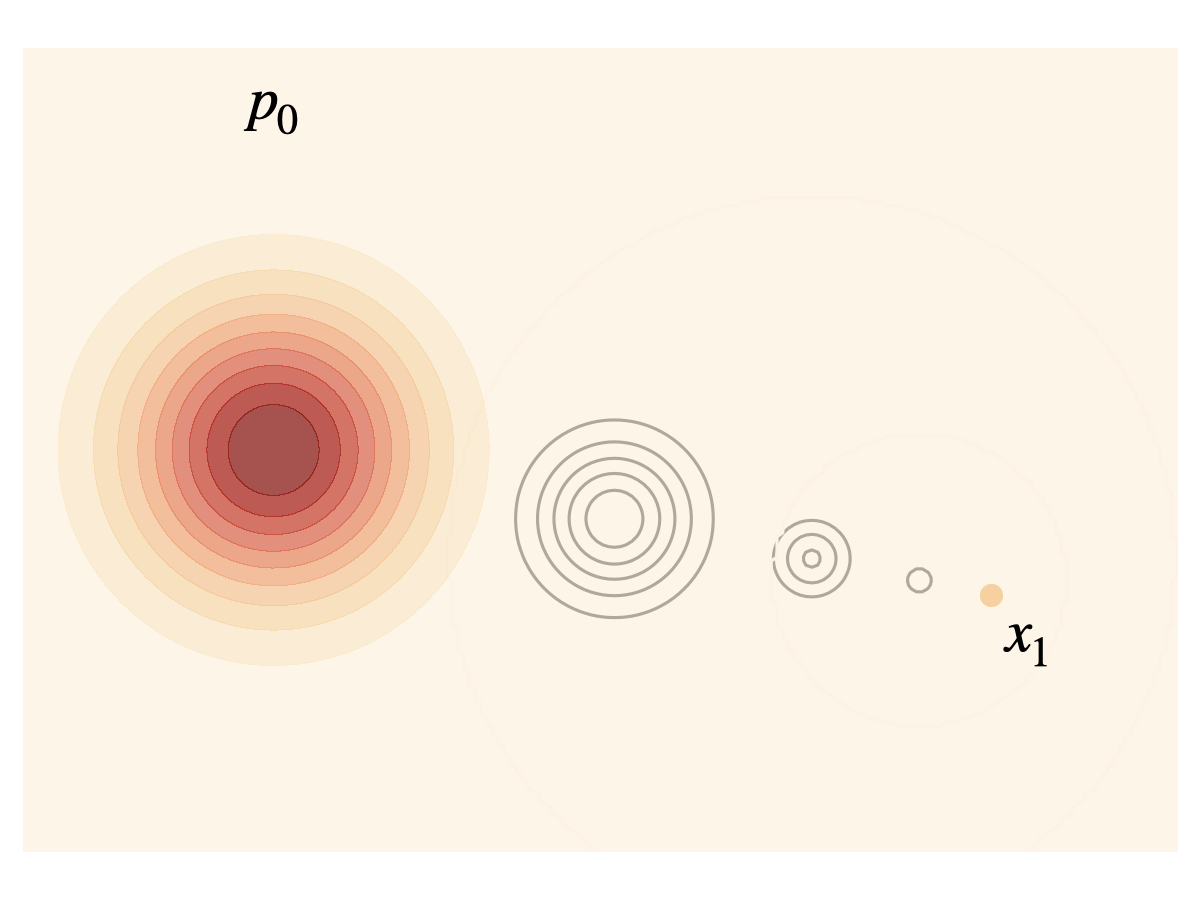} & \includegraphics[
                trim=0 50 0 40,  % left, bottom, right, top
                clip,
                width=0.31\textwidth
            ]{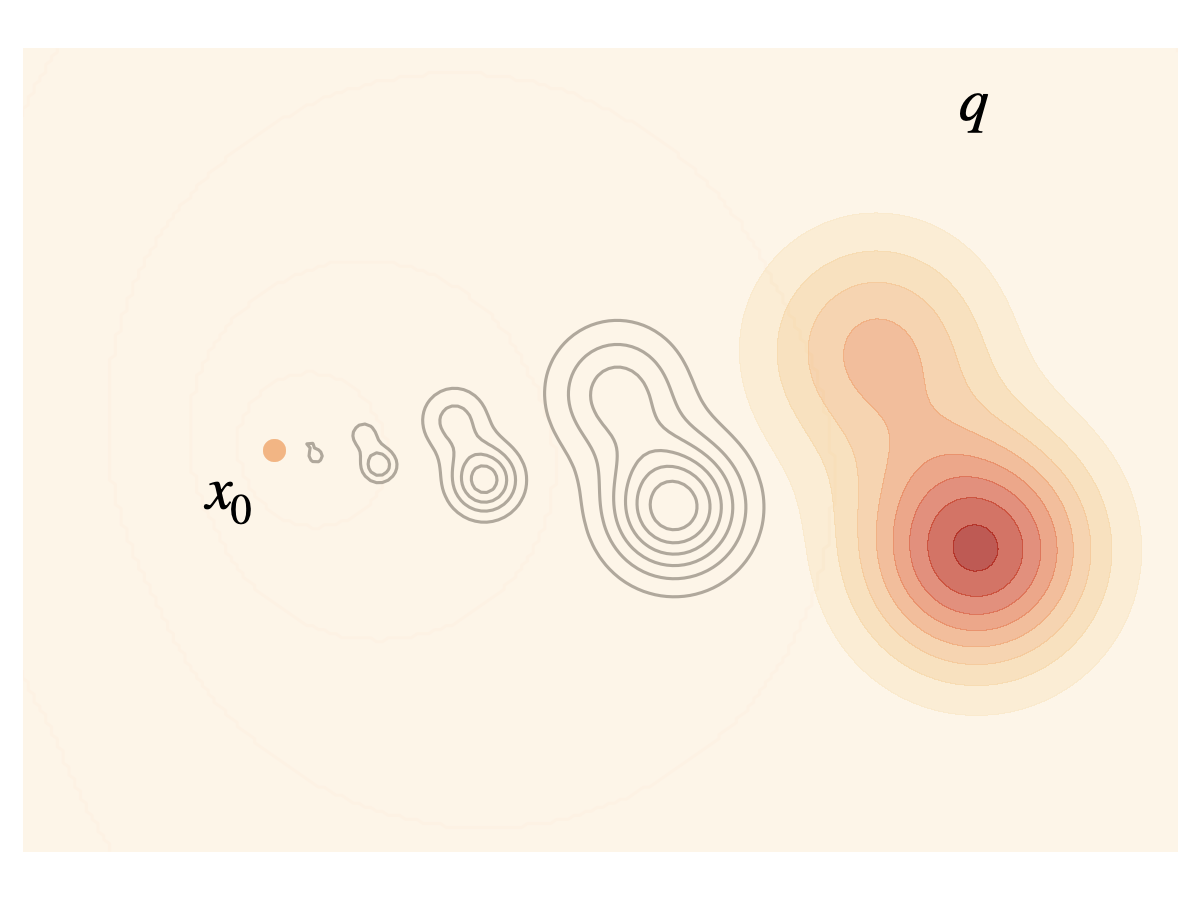}  &
            \includegraphics[
                trim=0 50 0 40,  % left, bottom, right, top
                clip,
                width=0.31\textwidth
            ]{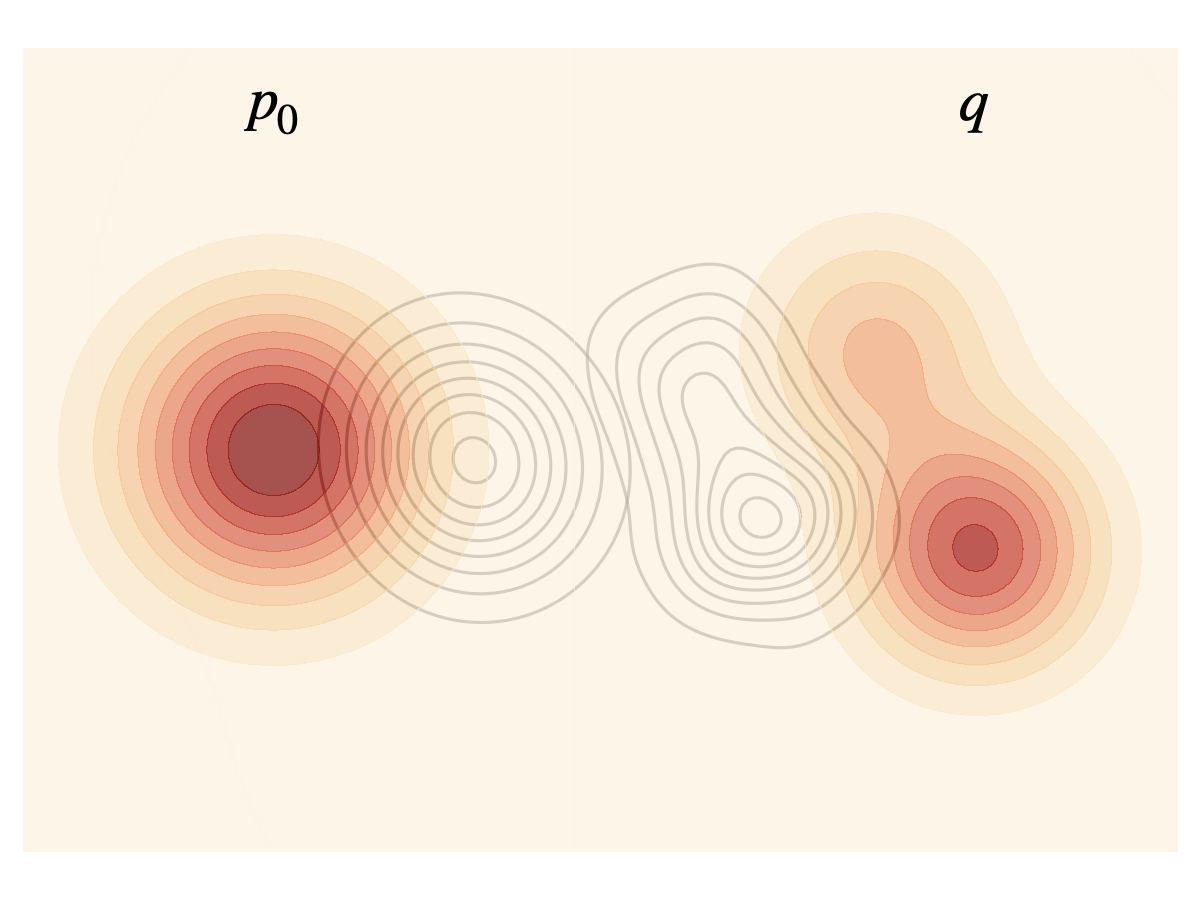}
    \end{tabular}
    \vspace{-5pt}
    \caption{
    (left) Flow matching conditional probability path is marginal of a noising process conditioned on a data point $x_1$. (middle) Flow sampling conditional probability path is a marginal of a denoising process conditioned on a noise point $x_0$. (right) The probability of the marginal process is the same in both cases.
    }
    \label{fig:fs_path_vs_fm_path}
    \vspace{-10pt}
\end{figure*}

\section{Preliminaries: Flow Matching}
A flow process $\parr{X_t}_{0\le t\le 1}$ is defined by a \emph{velocity field} $v:\sR^{d}\times[0,1]\too\sR^d$ and a boundary condition,
\begin{equation}\label{e:flow_process}
    \diff{X_t} = v_t\parr{X_t}\diff{t},\quad X_0 \sim p_0,
\end{equation}
where the \emph{source distribution} $p_0$ is the marginal of the process at time $t=0$. The \emph{probability path} $p_t$ is the marginal of the process at time $0 \le t \le 1$, and it is given by the \emph{continuity equation},
\begin{equation}\label{eq:continuity}
    \frac{\partial}{\partial t} p_t(x) + \nabla\brac{p_t(x)v_t(x)} = 0.
\end{equation}

Given a dataset of i.i.d. samples from a target distribution $q$ on $\R^d$, and an easy to sample source distribution $p_0$, the goal of flow matching (FM)~\citep{lipman2023flow,liu2023flow, albergo2023si} is to learn a velocity field $v^{\theta}_t$ such that the marginal of $X_t$ interpolate between the source and the target, \ie,
\begin{equation}
    X_0\sim p_0,\quad X_1\sim p_1\equiv q.
\end{equation}

To learn $v^{\theta}_t$, FM uses a supervising flow process conditioned at time $t=1$. Given $X_1=x_1$, the conditional velocity is,
\begin{equation}\label{eq:v_cond_fm}
    v_{t|1}(x|x_1) = \frac{\dot{\sigma}_t}{\sigma_t}\parr{x - \alpha_tx_1} + \dot{\alpha}_tx_1.
\end{equation}
The ODE~(\ref{e:flow_process}) for the conditional flow~(\ref{eq:v_cond_fm}) is solved by,
\begin{equation}\label{eq:x_cond_fm}
    X_{t} = \sigma_tX_0 + \alpha_t x_1,\quad X_0 \sim p_0.
\end{equation}
The marginal of the \emph{interpolant}~(\ref{eq:x_cond_fm}) $X_t$ is the \emph{conditional probability path} that is given by the push forward formula
\begin{equation}\label{eq:p_cond_fm}
    p_{t|1}(x|x_1) = \frac{1}{\sigma_t^d}p_0\parr{\frac{x -\alpha_t x_1}{\sigma_t}}.
\end{equation}
Finally, the training objective for the velocity is the FM loss
\begin{equation}\label{eq:loss_fm}
    \gL_{\text{FM}}\parr{\theta} = \E\brac{\norm{v_t^{\theta}\parr{X_{t}} - v_{t|1}\parr{X_{t}|X_1}}^2}, % \Big|\ X_1
\end{equation}
where $X_1\sim q$ sample is obtained from the dataset, and $X_{t}\sim p_{t|1}$ is sampled using the interpolant~(\ref{eq:x_cond_fm}).

Indeed, the generated probability path by the optimal $v^{\theta}_t$ is
\begin{equation}
    p_t(x) = \E\brac{p_{t|1}(x|X_1)},
\end{equation}
that equals the source, $p_0$ and target
$q$ at $t=0$ and $t=1$.

\vspace{-8pt}
\section{Flow Sampling}
In contrast to flow matching, we are interested in the data free settings, but our goal is similar. Given a score function of the target $\nabla\log q(x_1) = \nabla r(x_1)$, and an easy to sample source distribution $p_0$, our goal is to learn a diffusion process $\parr{X_t}_{0\le t\le 1}$ such that its probability path interpolates between source $p_0$ and target $p_1\equiv q$.

\vspace{-5pt}
\paragraph{Diffusion process} A diffusion process is defined by a drift $u:\R^d\times[0,1]\too\R^d$, a diffusion coefficient $g_t:[0,1]\rightarrow\R_{\ge0}$, and a boundary condition,
\begin{equation}\label{e:diffusion_process}
    \diff X_t = u_t\parr{X_t}\diff t + g_t  \diff B_t,\quad \ X_0\sim p_0.
\end{equation}
where $B_t$ is the Brownian motion. The probability path $p_t$ of a diffusion process is given by the Fokker–Planck equation,
\begin{equation}\label{eq:fokker_planck}
    \frac{\partial}{\partial t} p_t(x) + \nabla\brac{p_t(x)u_t(x)} = \frac{g_t^2}{2}\nabla^2p_t(x).
\end{equation}
We recall the standard correspondence between deterministic transport and a diffusion with the same probability path.
\begin{restatable}{proposition}{FlowToDiffusion}
\label{prop:flow_to_diffusion}
% \begin{proposition}\label{prop:flow_to_diffusion}
    Let $v_t$ be a velocity that defines the flow process~(\ref{e:flow_process}) with a marginal $p_t$. Then the drift
    \begin{equation}\label{eq:v_to_u}
        u_t(x) = v_t(x) + \frac{g_t^2}{2}\nabla\,\log p_t(x),
    \end{equation}
    defines a diffusion process~(\ref{e:diffusion_process}) with the same marginal $p_t$.
% \end{proposition}
\end{restatable}
Proposition \ref{prop:flow_to_diffusion} is a known result~\citep{ANDERSON1982313}; we defer the proof to the appendix. This observation makes explicit how score information enters the drift, which is exactly the mechanism we will exploit to inject target information.

\paragraph{Key Idea} Given a data sample $X_1\sim q$, the conditional flow~(\ref{eq:v_cond_fm}) which is the supervision for the FM loss~(\ref{eq:loss_fm}), is independent of the target distribution $q$. Thus, all learning signals for the target $q$ come through marginalization.

In data-free settings, we don't have access to data samples $X_1\sim q$, hence we cannot marginalize over the target $q$. Our key observation is that we can condition on samples from the source $X_0\sim p_0$, move all the signal about the target $q$ to a supervising denoising diffusion process conditioned at time $t=0$ by $X_0$, and marginalize over the source $p_0$.

We use a conditional probability path defined as a push forward of the target distribution
\begin{equation}\label{eq:p_cond_fs}
    p_{t|0}\parr{x|x_0} = \frac{1}{\alpha_t^d}q\parr{\frac{x -\sigma_t x_0}{\alpha_t}}.
\end{equation}
This is in contrast to flow matching and diffusion models~\cite{ho2020denoising} that pushes the source distribution~(\ref{eq:p_cond_fm}). Still, in both cases the marginal probability path is the same,% for all $t\in [0,1]$,
\begin{equation}
    \E\brac{p_{t|0}\parr{x|X_0}} = p_t(x) = \E\brac{p_{t|1}\parr{x|X_1}}.
\end{equation}
The difference between pushing forward the source $p_0$ or the target $q$ is  illustrated in Figure \ref{fig:fs_path_vs_fm_path}. Notably, the former is independent of the target, and the latter is a scale and shift of the target.

\textbf{Supervising drift} \ \ Importantly, we identify two processes that generate the conditional path in \eqref{eq:p_cond_fs}. The first, is defined by a conditional velocity similar to FM (\ref{eq:v_cond_fm}) but conditioned on time $t=0$. That is,
\begin{equation}\label{eq:v_cond_fs}
    v_{t|0}(x|x_0) = \frac{\dot{\alpha}_t}{\alpha_t}\parr{x - \sigma_tx_0} + \dot{\sigma}_t x_0,
\end{equation}
which generates the probability path $p_{t|0}$~(\ref{eq:p_cond_fs}) since they satisfy the continuity equation (\ref{eq:continuity}). The second, is defined by the conditional drift
\begin{equation}\label{eq:cond_drift}
    u_{t|0}(x|x_0) := v_{t|0}(x|x_0) + \frac{g_t^2}{2}\,\nabla\log p_{t|0}\parr{x|x_0},
\end{equation}
which by Proposition~(\ref{prop:flow_to_diffusion}), generates the same conditional probability path $p_{t|0}$~(\ref{eq:p_cond_fs}).

The two processes are important. The second process, \ie, the drift~(\ref{eq:cond_drift}) allows to sample $X_1\sim q$ given $X_0=x_0$ using simulation, and regressing onto the drift $u_{t|0}$ suffices to learn a sampler of $q$. The first process, \ie. the velocity~(\ref{eq:v_cond_fs}) is solved in closed form by the interpolant~(\ref{eq:x_cond_fm}) and allows efficient reuse of existing samples $X_1\sim q$ with simulation free sampling of $p_{t|0}$, significantly reducing training cost.

\paragraph{Training objective} Similar to FM, we use the marginalization trick~\citep{lipman2024guide} and train our drift $u^{\theta}$ by regressing onto the conditional drift $u_{t|0}$ (\ref{eq:cond_drift}). This yields a \emph{diffusion matching} (DM) objective
\begin{equation}\label{eq:diffusion_matching}
    \gL_{\text{DM}}\parr{\theta} = \E\brac{\norm{u^{\theta}(X_{t}) - u_{t|0}\parr{X_{t}|X_0}}^2},
\end{equation}
where $X_0\sim p_0$ and $X_{t}\sim p_{t|0}$. The conditional $p_{t|0}$ is sampled using the flow process~(\ref{e:flow_process}) defined by $v_{t|0}$~(\ref{eq:v_cond_fs}) which can be done efficiently since it is solved by the interpolant,

\begin{equation}\label{eq:x_cond_fs}
    X_{t} = \sigma_t X_0 + \alpha_t X_1, \quad X_0 \sim p_0,\ X_1 \sim q.
\end{equation}

\begin{myframe}
\begin{restatable}{proposition}{Drift}
\label{prop:drift}
Let $X_0=x_0$, and $v_{t|0}$ the conditional velocity~(\ref{eq:v_cond_fs}) that generates $p_{t|0}$ the conditional probability path~(\ref{eq:p_cond_fs}). Then, for every $g_t \geq 0$ and $X_1\sim q$,
\begin{align}\label{eq:cond_drift_path}
    u_{t|0}\parr{X_{t}|x_0} = \dot{\alpha}_tX_1 + \dot{\sigma}_tx _0 + \frac{g_t^2}{2 \alpha_t}\nabla r\parr{X_1}.
\end{align}
where $X_t$ is the interpolant~(\ref{eq:x_cond_fs}) defined by $v_{t|0}$.
\end{restatable}
\end{myframe}

Proof of Proposition~\ref{prop:drift} is in Appendix~\ref{a:proofs}. This proposition is extremely useful, since it implies that for every sample from the target $X_1\sim q$, the score $\nabla\,\log q(X_1)=\nabla\,r(X_1)$ needs to be evaluated only once, to evaluate $u_{t|0}$, the conditional drift~(\ref{eq:cond_drift_path}), along $X_t$ the interpolant~(\ref{eq:x_cond_fs}) for all $t\in[0,1]$ and $X_0\sim p_0$. Hence, we can use a \emph{replay buffer} $\gB$ to store the pairs $X_1, \nabla\,r\parr{X_1}$ and reuse them during training. This improves training efficiency significantly in cases where the reward is costly to compute.

\paragraph{Linear process}
We train our model using the known linear scheduler~\citep{lipman2023flow, liu2022rectified},
\begin{equation}\label{eq:linear_process}
    \alpha_t = t,\quad \sigma_t =1-t, \quad g_t^2 = 2 \gamma t, \quad \gamma >0.
\end{equation}
The squared diffusion coefficient is chosen $g^2_t\propto \alpha_t$ to account for the possible singularity at time $t=0$ as implied by Proposition~(\ref{prop:drift}). This yields a stable target for the regression, and  $\gamma$ is a hyper-parameter that is used to further regularize the target during training.

\subsection{Diffusion matching with the Flow Sampling loss}
Sampling the target $X_1\sim q$ using simulation with the conditional drift~(\ref{eq:cond_drift}) requires repeated evaluations of the energy gradient $\nabla r(x)$, which we want to avoid.

\paragraph{Fixed point iteration}
To overcome this, we propose a fixed point iteration training algorithm. We defined $X_1^{\bar{\theta}}$ to be samples from the detached current state of the model, \ie, $\bar{\theta}=\text{stopgrad}(\theta)$, using the Euler--Maruyama algorithm,
\begin{equation}
    X_{t+h}^{\bar{\theta}} = X_t^{\bar{\theta}} +hu_t^{\bar{\theta}}\parr{X_t^{\bar{\theta}}} + \sqrt{2\gamma t h}Z_t,
\end{equation}
where $Z_t\sim \gN\parr{0,I}$. Then, by minimizing the DM loss~\eqref{eq:diffusion_matching}, but sampling $X_1^{\bar{\theta}}\sim p_1^{\bar{\theta}}$ instead of $X_1\sim q$ during training we obtain our flow sampling objective.
% \vspace{-5pt}
\begin{algorithm}[H]
\caption{Flow sampling training}
\label{alg:fs_train}
\begin{lstlisting}[language=ModernPython]
# model_fn - trainable drift model
# solver_fn - Euler--Maruyama solver
# score_fn - gradient of the reward
# buffer - replay buffer
# gam - diffusion coefficient
# bz, shape - batch size and data shape
for _ in range(outer_loops):
    # exploration phase
    x_0 = randn(bz, *shape)
    with no_grad():
        x_1 = solver_fn(model_fn, gam, x_0)
        score_1 = score_fn(x_1)
    buffer.push(x_1, score_1)

    # optimization phase
    for _ in range(inner_loops):
        x_1, score_1 = buffer.sample(bz)
        x_0 = randn(bz, *shape)
        t = rand(bz, *[1]*len(shape))

        x_t = (1-t)*x_0 + t*x_1
        u_t = x_1 - x_0 + gam*score_1
        pred = model_fn(x_t, t)
        loss = mse_loss(pred, u_t)

        loss.backward()
        optimizer.step()
        optimizer.zero_grad()
\end{lstlisting}
\end{algorithm}

\yellowboxuncenter{
\paragraph{Flow Sampling Objective}
We train $u^\theta_t$ by regressing onto a closed-form target~(\ref{eq:cond_drift_path}), and linear process~(\ref{eq:linear_process}). The target distribution is approximated with samples of the current detached state of the model $X_1^{\bar{\theta}}\sim p_1^{\bar{\theta}}$.
\begin{equation}\label{eq:loss_fs}
    \gL_{\text{FS}}\parr{\theta} = \E\brac{\norm{u_t^{\theta}(X_{t}) - \parr{X_1^{\bar{\theta}} - X_0 + \gamma\nabla r\large(X_1^{\bar{\theta}}\large)}}^2},
\end{equation}
where $X_t$ is the linear interpolant process
\begin{equation}
    X_{t} = (1-t)X_0 + t X_1^{\bar{\theta}},\,\, X_0 \sim p_0,\,\, X_1^{\bar{\theta}}\sim p_1^{\bar{\theta}}.
\end{equation}
}
We use the same alternating scheme as Adjoint Sampling~\cite{AS}, which consists of two phases:
\begin{enumerate}[leftmargin=*, nosep]
    \item \textbf{Exploration phase:} Using the current detached model $u^{\bar{\theta}}$, we draw a batch of samples $X_1^{\bar{\theta}}\sim p_1^{\bar{\theta}}$ with the Euler--Maruyama solver, evaluate the score $\nabla r\parr{X_1^{\bar{\theta}}}$, and store the pairs in the replay buffer $\gB \gets \gB\cup\set{ X_1^{\bar{\theta}},\ \nabla r\parr{X_1^{\bar{\theta}}} }$.

    \item \textbf{Optimization phase:} Fixing the state of the replay buffer $\gB$, we iterate over it $\set{ X_1^{\bar{\theta}},\ \nabla r\parr{X_1^{\bar{\theta}}} }\sim\gB$, and perform multiple gradient steps on $\gL_{\mathrm{FS}}\parr{\theta}$~(\ref{eq:loss_fs}).
\end{enumerate}
\newpage
The two phases are repeated until convergence. The pseudo code for our training method is given in Algorithm~\ref{alg:fs_train}, where a Gaussian source $p_0=\gN\parr{0,I}$ is assumed. The number of energy gradient $\nabla r(x)$ evaluations equals the number of $X_1^{\bar{\theta}}\sim p_1^{\bar{\theta}}$  acquired
during training~\citep{AS}.

\section{Extension to Riemannian Manifolds}

A strong advantage of our flow sampling formulation for training diffusion samplers~\citep{akhoundsadegh2024idem, vargas2023dds, AS} is that it naturally extends diffusion samplers to non-Euclidean geometries.

This section describes how to adapt flow sampling to Riemannian Manifolds. Observing our FS loss~(\ref{eq:loss_fs}) we identify three key requirements: (i) simulation of  diffusion process on manifolds~\citep{huang2022riemannian} for sampling, (ii) an interpolant process~(\ref{eq:x_cond_fs}) on manifolds~\citep{chen2023riemannian}, and (iii) a conditional drift~(\ref{eq:cond_drift_path}) on manifolds that generate the same marginal as the interpolant.

\textbf{Setup} \ \
Let an ambient space $\gA=\R^{d+1}$ equipped with a metric $\inner{\cdot}{\cdot}_{\Sigma}$ given by  $\Sigma\in \R^{(d+1)\times (d+1)}$ an invertible symmetric matrix. We assume a complete, connected, smooth Riemannian manifold with a constant curvature $\kappa$ given by
\begin{equation}\label{eq:manifold}
    \gM = \set{x\in\gA\ |\ \inner{x}{x}_{\Sigma}=\frac{1}{\kappa}}
\end{equation}
Note, in case $\gM$ has more than one connected component, we treat each one separately. The tangent space $T_x\gM$ at a point $x\in\gM$ is defined as
\begin{equation}\label{eq:tangent_space}
    T_x\gM = \set{w \in \gA\ |\ \inner{w}{x}_{\Sigma}=0 }.
\end{equation}

\vspace{-5pt}
\textbf{Diffusion on manifolds} \ \
As in the Euclidean case, a diffusion process is defined by a drift $u_t(x)\in T_x\gM$, and a boundary condition,
\begin{equation}\label{eq:manifold_sde_strat}
\mathrm{d}X_t = u_t(X_t)\,\mathrm{d}t + \sqrt{2\gamma t} P^{\perp}_{X_t}\circ \mathrm{d}B_t,\quad X_0 \sim p_0.
\end{equation}
where, $B_t$ is the standard Brownian motion in the ambient space $\gA$, the diffusion coefficient is $g_t := \sqrt{2\gamma t}$, and $P_{x}^{\perp}$ is an orthogonal projection defined for any $x, w\in \gA$ as
\begin{equation}\label{eq:orthogonal_projection}
    P_{x}w = \frac{\inner{x}{w}_{\Sigma}}{\inner{x}{x}_{\Sigma}}x,\quad P^{\perp}_{x}w = w - P_{x}w.
\end{equation}
Thus, $P^{\perp}_{X_t}\circ \mathrm{d}B_t \in T_{X_t}\gM$. The marginal of the process $p_t$ is defined by the Riemannian Fokker-Planck equation
\begin{equation}\label{eq:fokker_plank_riem}
\partial_t p_t = -\mathrm{div}_{\gM}(p_t u_t) + \gamma t \Delta_{\gM}p_t,
\end{equation}
where $\mathrm{div}_{\gM}$ and $\Delta_{\gM}$ are the divergence and Laplacian operators (resp.) generalized to manifolds.

Lastly, the Euler--Maruyama solver on the manifold is
\begin{equation}\label{eq:euler_maruyama_riem}
X_{t+h} = \exp_{X_t}\parr{h u_t(X_t) + \sqrt{2\gamma th} P^{\perp}_{X_t} Z_t},
\end{equation}
where $Z_t\sim\gN\parr{0, I}$, and $\exp_{x}:T_x\gM\too\gM$ is the exponential map of $\gM$. For our set of constant curvature manifolds~(\ref{eq:manifold}),
%defined by the metric $\Sigma$ and sectional curvature $\kappa$,
the exponential map can be solved in closed form~\cite{chen2023riemannian} that is dependent on $\Sigma$ and $\kappa$.

% \paragraph{Geodesic Interpolant}
\textbf{Geodesic Interpolant} \ \
Given $X_0=x_0\in\gM$ and $X_1\sim q$, similar to Riemannian flow matching (RFM)~\citep{chen2023riemannian}, instead of an interpolant~(\ref{eq:x_cond_fs}) between $X_0$ and $X_1$, we use the geodesic,
\begin{equation}\label{eq:x_cond_fs_riem}
    X_t = \phi_t\parr{X_1|x_0}= \exp_{X_1}\brac{(1-t)\log_{X_1}(x_0)},
\end{equation}
where $\log_{x}:\gM\too T_x\gM$ is the logarithmic map which can be solved in closed form for the set of manifolds we assumed. The marginal of the geodesic interpolant is the conditional probability path,
\begin{equation}\label{eq:p_cond_fs_riem}
    p_{t|0}(x|x_0) = q(\phi_{t|0}^{-1}(x|x_0))\,\big|\det\nolimits_{\gM} D\phi_{t|0}^{-1}(x|x_0) \big|.
\end{equation}
% \paragraph{Conditional drift}
\textbf{Conditional drift} \ \
Similar to the Euclidean case, we find a conditional drift $u_{t|0}$ that generate the conditional path~(\ref{eq:p_cond_fs_riem}) $p_{t|0}$. Evaluated at the geodesic interpolant~(\ref{eq:x_cond_fs_riem}), the drift is
\begin{equation}\label{eq:u_cond_riem}
    u_{t|0}(X_t|x_0) = \dot X_t + \gamma t \nabla_{\gM}\log p_t(X_t|x_0),
\end{equation}
where $\dot X_t =\frac{d}{dt}\phi_t\parr{X_1|x_0}$ is the derivative w.r.t. time of the geodesic, and the conditional score decomposes to
\begin{align}\label{eq:score_cond_riem}
\nabla_{\gM}\log p_{t|0}(X_t|x_0) =& (J_t^{-1})^*\,\nabla_{\gM} r(X_1)\\
 &- \nabla_{\gM}\log |\det\nolimits_{\gM} \parr{J_t}|,
\end{align}
% where $J_t:T_{X_1}\gM\too T_{X_t}\gM$,
where $J_t = D\phi_t(X_1|x_0)$ is the Jacobian of the geodesic, and $(J_t^{-1})^*$ is its inverse adjoint.
\begin{myframe}
\begin{restatable}{proposition}{JClosedFormProp}
\label{prop:J_closed_form}
% \begin{proposition}\label{prop:J_closed_form}
    For Riemannian manifold $\gM$ defined in \eqref{eq:manifold} with a metric $\Sigma$ and constant curvature $\kappa$, the Jacobian of the geodesic w.r.t. $X_1$ is
    \begin{equation*}
        J_t = tT_{X_1\too X_t}P_{\dot{X}_1} +c_t\parr{X_1, x_0}T_{X_1\too X_t}P_{\dot{X}_1}^{\perp},
    \end{equation*}
    where $T_{X_1\too X_t}:T_{X_1}\gM\too T_{X_t}\gM$ is the parallel transport,  $c_t\parr{X_1, x_0}$ is a time dependent scaling,
    \begin{equation}
    c_t(X_1, x_0) = \begin{cases}
        \dfrac{\sin(t\omega_1\sqrt{\kappa})}{\sin(\omega_1\sqrt{\kappa})} & \text{if } \kappa > 0 \\[2ex]
        \dfrac{\sinh(t\omega_1\sqrt{|\kappa|})}{\sinh(\omega_1\sqrt{|\kappa|})} & \text{if } \kappa < 0
    \end{cases},
\end{equation}
and $\omega_1 = \norm{\dot{X}_1}_{\Sigma}$ the geodesic distance of $X_1$ and $x_0$.
% \end{proposition}
\end{restatable}
\end{myframe}
Proposition~(\ref{prop:J_closed_form}) allows us to compute the conditional drift~(\ref{eq:u_cond_riem}) in closed form. As it is composed of rank-1 matrices, it is efficient to evaluate. Its proof is based on results in Jacobi fields~\citep{lee2018introduction}, and given in Appendix~\ref{a:proofs}.

\yellowboxuncenter{
\paragraph{Riemannian Flow Sampling objective} The model $u_t^{\theta}$ is trained by regressing onto $u_{t|0}$ the conditional drift~(\ref{eq:u_cond_riem}) that is given in closed form in Proposition~\ref{prop:J_closed_form}.
\begin{equation}
    \gL_{\text{RFS}}\parr{\theta} = \E\brac{\norm{P_{X_t}^{\perp}u_t^{\theta}\parr{X_t} - u_{t|0}\parr{X_t|X_0}}^2_{\Sigma}}
\end{equation}
where $X_t$ is the geodesic interpolant~(\ref{eq:x_cond_fs_riem}), and the target distribution is approximated with $X_1^{\bar{\theta}}\sim p_1^{\bar{\theta}}$ samples of the current detached state of the model using the Riemannian Euler-Maruyama solver~(\ref{eq:euler_maruyama_riem}). Optimization is done using the adjusted to manifold version of Algorithm~\ref{alg:fs_train}.
}

%%%%%
\subsection{Flow sampling on a hyper-sphere}\label{sec:hyper_sphere}
As an example, we consider the case of a hyper-sphere $\gM = \sS^{d}$ given by setting the metric and the curvature to
\begin{equation}
    \Sigma= I,\quad \kappa=1.
\end{equation}
Given $X_0=x_0\in\sS^d$, and $X_1\sim q$, the geodesic distance is
\begin{equation}
    \omega_1 =\norm{\log_{X_1}\parr{x_0}}= \arccos\parr{X_1^\top x_0}.
\end{equation}
The derivative w.r.t. time of the geodesic interpolant~(\ref{eq:x_cond_fs_riem}) at time $t=1$ is
\begin{equation}
    \dot{X}_1 = -\log_{x_1}\parr{x_0}= \omega_1\frac{\cos{\omega_1}X_1 - x_0}{\sin{\omega_1}},
\end{equation}
and, the geodesic interpolant~(\ref{eq:x_cond_fs_riem}) is the SLERP function,
\begin{align}\label{eq:slerp}
    X_t &= \cos\parr{{(1-t)\omega_1}}X_1 - \sin\parr{(1-t)\omega_1}\frac{\dot{X}_1}{\omega_1},
\end{align}
and its derivative w.r.t. to time $t$ is
\begin{align}\label{eq:slerp_derivative}
    \dot{X}_t &= \omega_1\sin\parr{{(1-t)\omega_1}}X_1 + \cos\parr{(1-t)\omega_1}\dot{X}_1.
\end{align}
The parallel transport from $X_1$ to $X_t$ on the hyper-sphere is the Householder reflection about the midpoint $X_1 + X_t$,
\begin{equation}\label{eq:parallel_transport}
    T_{X_1\too X_t} = I-2P_{X_1 + X_t}.
\end{equation}
Since the parallel transport $T_{X_1\too X_t}$ is an orthogonal operator, using Proposition~\ref{prop:J_closed_form} the inverse Jacobian adjoint is
\begin{equation}\label{eq:slerp_inv_J_adj}
    \parr{J_t^{-1}}^* = \frac{1}{t}T_{X_1\too X_t}P_{\dot{x}_1}+ \frac{\sin\parr{\omega_1}}{\sin\parr{t\omega_1}}T_{X_1\too X_t}P_{\dot{X}_1}^\perp.
\end{equation}
This implies the correction to the conditional score is
\begin{equation}\label{eq:slerp_correction}
    \nabla_{\gM}\log\det\parr{J_t} = (d-1)\parr{ t\cot{t\omega_1} - \cot{\omega_1}}\frac{\dot{X}_1}{\omega_1}.
\end{equation}
Hence we have all terms in the conditional drift~(\ref{eq:u_cond_riem}) in closed form. Importantly, both the projection operator~(\ref{eq:orthogonal_projection}) and the parallel transport operator~(\ref{eq:parallel_transport}) include only rank 1 operators, thus are computationally efficient.

\begin{table*}[t]
\centering
\begin{minipage}{\linewidth}
\caption{Results for the synthetic energy function experiments. We report a geometric $\mathcal{W}_2$ metric~\citet{klein2024equivariant} and 1D energy histogram $E(\cdot)\,\mathcal{W}_2$ metric with respect to ground truth samples. See~\Cref{app:w2_metric} for more details. We use light and dark shading to denote best and 2nd best result respectively. \textsuperscript{\textdagger} We include baseline measurements over random subsets of the ground truth samples. This can be considered the ``optimal'' value for that metric, which was not clear in previous instantiations of this benchmark. Flow Sampling XL results are omitted for DW4 due to minimal improvement. Mean and standard error are reported across 5 training runs.
}
\centering
% \begin{minipage}{\linewidth}
%\resizebox{\textwidth}{!}{%
\renewcommand{\arraystretch}{1.0}
\setlength{\tabcolsep}{3pt}
% \resizebox{\textwidth}{!}{%

\begin{tabular}{@{} l cc cc cc @{}}
\toprule
 & \multicolumn{2}{c}{DW-4 $(N=4, d=2)$}      & \multicolumn{2}{c}{LJ-13 $(N=13, d=3)$}    & \multicolumn{2}{c}{LJ-55 $(N=55, d=3)$}
\\
\cmidrule(lr){2-3} \cmidrule(lr){4-5} \cmidrule(lr){6-7}
Method &$\mathcal{W}_2$ $\downarrow$ & $E(\cdot)$ $\mathcal{W}_2$ $\downarrow$
& $\mathcal{W}_2$ $\downarrow$ & $E(\cdot)$ $\mathcal{W}_2$ $\downarrow$
& $\mathcal{W}_2$ $\downarrow$ & $E(\cdot)$ $\mathcal{W}_2$ $\downarrow$ \\ \hline
PIS {\tiny\citep{zhang2022path}} & 0.68{\color{gray}\tiny$\pm$0.23} & 0.65{\color{gray}\tiny$\pm$0.25} &
 1.93{\color{gray}\tiny$\pm$0.07} & 18.02{\color{gray}\tiny$\pm$1.12} &
 4.79{\color{gray}\tiny$\pm$0.45}  & 228.70{\color{gray}\tiny$\pm$131.27} \\
DDS {\tiny\citep{vargas2023denoising}} &
 0.92{\color{gray}\tiny$\pm$0.11} & 0.90{\color{gray}\tiny$\pm$0.37} &
 1.99{\color{gray}\tiny$\pm$0.13} & 24.61{\color{gray}\tiny$\pm$8.99} &
 4.60{\color{gray}\tiny$\pm$0.09}  & 173.09{\color{gray}\tiny$\pm$18.01} \\
iDEM {\tiny\citep{akhound2024iterated}} &
 0.70{\color{gray}\tiny$\pm$0.06} & 0.55{\color{gray}\tiny$\pm$0.14} &
  1.61{\color{gray}\tiny$\pm$0.01} & 30.78{\color{gray}\tiny$\pm$24.46} &
 4.69{\color{gray}\tiny$\pm$1.52} & 93.53{\color{gray}\tiny$\pm$16.31} \\
AS {\tiny\citep{AS}} &
  0.62{\color{gray}\tiny$\pm$0.06} & 0.55{\color{gray}\tiny$\pm$0.12} &
 1.67{\color{gray}\tiny$\pm$0.01} &  2.40{\color{gray}\tiny$\pm$1.25} &
  4.04{\color{gray}\tiny$\pm$0.05} & 30.83{\color{gray}\tiny$\pm$8.19}\\
ASBS {\tiny\citep{liu2025asbs}} &
  \cellhi 0.43{\color{gray}\tiny$\pm$0.05} &  \cellhi 0.20{\color{gray}\tiny$\pm$0.11} &
 1.59{\color{gray}\tiny$\pm$0.03} &  1.99{\color{gray}\tiny$\pm$1.01} &
  4.00{\color{gray}\tiny$\pm$0.03} &  28.10{\color{gray}\tiny$\pm$8.15} \\
\hline
\textbf{Flow Sampling} &
 \cellhii 0.36{\color{gray}\tiny$\pm$0.03} & \cellhii 0.11{\color{gray}\tiny$\pm$0.04} &

\cellhi 1.58{\color{gray}\tiny$\pm$0.01} &
 \cellhi 0.97{\color{gray}\tiny$\pm$0.53} &
 \cellhi 3.98{\color{gray}\tiny$\pm$0.01} &\cellhi 21.32{\color{gray}\tiny$\pm$0.63} \\
 \textbf{Flow Sampling XL} &
 -- & -- &
 \cellhii 1.57{\color{gray}\tiny$\pm$0.01} & \cellhii 0.81{\color{gray}\tiny$\pm$0.26} &
 \cellhii 3.96{\color{gray}\tiny$\pm$0.01} & \cellhii 16.29{\color{gray}\tiny$\pm$0.93} \\
 \hline
 \addlinespace[0.15em]
 {\color{gray} Subsets of MCMC (ground truth)\textsuperscript{\textdagger}} & {\color{gray}0.31\tiny$\pm$0.01} & {\color{gray}0.10\tiny$\pm$0.03} &{\color{gray}1.57\tiny$\pm$0.01} & {\color{gray}0.35\tiny$\pm$0.03}& {\color{gray}3.85\tiny$\pm$0.05} &
 {\color{gray}0.57\tiny$\pm$0.05}\\
\bottomrule
\end{tabular}
% \end{minipage}
%}
\label{tab:synthetic}
\end{minipage}
\vspace{-10pt}
\end{table*}
\section{Related Works}
\textbf{Learning-Augmented Classical Samplers} \ \
MCMC and SMC methods provide asymptotically exact samples from unnormalized distributions, but can mix slowly or require many target evaluations in high-dimensional, multimodal settings. Learning-augmented variants
replace hand-designed proposals with normalizing flows or learnable
stochastic dynamics~\citep{albergo2019flow, arbel2021annealed,
gabrie2022adaptive, matthews2022continual, albergo2024nets,
holderrieth2025leaps}, typically requiring importance weighting or
accept--reject corrections that drive up energy-evaluation cost.

\textbf{Diffusion Samplers} \ \
Diffusion samplers adapt score matching to unnormalized targets via
auxiliary corrections (MCMC, importance weighting, or resampling), each
inflating energy-evaluation cost~\citep{phillips2024particle,
de2024target, chen2024sequential}.
PITA~\citep{akhoundprogressive} trains a temperature-annealed ladder of
diffusion samplers, but requires high-temperature MCMC data to
bootstrap. Off-policy variants~\citep{malkin2023trajectory,
richterimproved} use trajectory-level objectives over a replay buffer
but typically require parameterizing the model in terms of the energy
gradient~\citep{he2025no}. iDEM~\citep{akhound2024iterated} instead
estimates the target score by Monte Carlo along the noising path,
incurring many energy evaluations per step to control variance.

\textbf{Stochastic Optimal Control and Adjoint Sampling} \ \
A closely related line of work formulates sampling from unnormalized
distributions as a stochastic optimal control problem, closely
connected to Schr\"odinger bridges~\citep{zhang2022path,
vargas2023denoising}, leading to bridge-based diffusion samplers and
transport objectives~\citep{berner2023optimal, richterimproved,
vargas2023transport, chen2024sequential, AS, liu2025asbs}. \cite{nam2025enhancing} further incorporates molecular inductive bias through well-tempered collective-variable biasing.
Flow Sampling is inspired by this line of work but departs from SOC
formulations by avoiding optimal couplings, yielding a simpler
training objective that supports non-memoryless noise schedules and
arbitrary source distributions without requiring an additional corrector network. In Appendix~\ref{app:as_connection} we show that with a Brownian-bridge
supervising path and matching Gaussian source, the Flow Sampling and
Adjoint Sampling regression targets agree in conditional expectation.

\textbf{Reward-tilting and Fine-Tuning} \ \
RL-style fine-tuning methods such as GRPO~\citep{shao2024deepseekmath}
and Flow-GRPO~\citep{liu2025flow} optimize a pretrained generative
model toward high-reward samples via policy gradients, without
explicitly targeting a particular distribution. A complementary line
instead targets the reward-tilted distribution—the base reweighted by
an exponential reward.
Concurrent to our work, Tilt Matching~\citep{potaptchik2025tilt}
regresses a velocity field that anneals from a base toward the tilted
target using only reward evaluations, while Meta Flow
Maps~\citep{potaptchik2026meta} amortize one-step posterior sampling
across intermediate noise levels to enable steering and off-policy
fine-tuning.

\begin{figure*}
\centering
    \begin{minipage}{0.25\linewidth}
        \centering
        \includegraphics[width=0.8\linewidth]{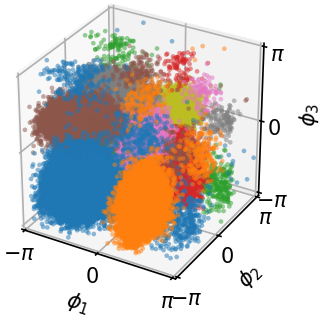}
        \subcaption*{ASBS}
    \end{minipage}
    \begin{minipage}{0.25\linewidth}
    \centering
    \includegraphics[width=0.8\linewidth]{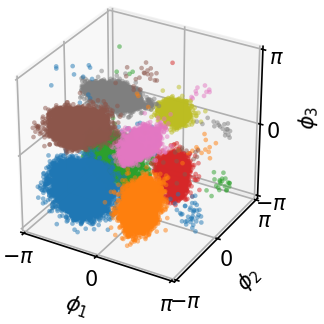}
        \subcaption*{Flow Sampling}
    \end{minipage}
    \begin{minipage}{0.25\linewidth}
    \centering
        \includegraphics[width=0.8\linewidth]{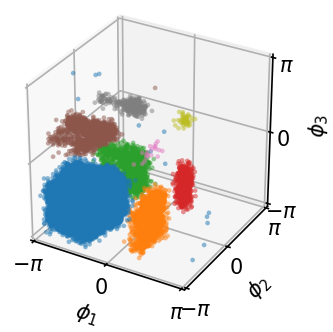}
        \subcaption*{MD}
    \end{minipage}
    \caption{Ala4 molecular structures represented by torsion coordinates $(\phi_1,\phi_2,\phi_3)$ with $2^3$ distinct modes.}\label{fig:ala4_torsions}
\end{figure*}

\vspace{-1em}
\section{Experiments}
\subsection{Synthetic Energy Functions}
We evaluate Flow Sampling on standard synthetic $n$-particle energy benchmarks \citep{kohler2020equivariant, midgley2023flow, klein2024equivariant, akhound2024iterated}: a 2D 4-particle Double-Well potential (DW-4), a 3D 13-particle Lennard--Jones system (LJ-13), and a 55-particle Lennard--Jones system (LJ-55).
The states are particle positions $x \in \mathbb{R}^{n \times d}$, corresponding to  $\texttt{shape}=(N,d)$ in Algorithm~1 and reward $r(x) = -E(x)$, where $E$ is the energy defined by the potentials LJ or DW. Details in Appendix~\ref{app:synthetic_energy}. We report geometric and energy $\mathcal{W}_2$ distances against long-run MCMC reference samples \citep{klein2024equivariant}.

\textbf{Baselines } We compare against iDEM \citep{akhound2024iterated}, PIS \citep{zhang2022path}, DDS \citep{vargas2023denoising}, and the adjoint-based solvers AS~\cite{AS} and ASBS~\cite{liu2025asbs}.
All methods use a 5-layer EGNN \citep{satorras2021n}, with a 10-layer variant (Flow Sampling XL) evaluated for our method. We use the harmonic prior as in ASBS~\cite{jing2023eigenfold}. Additionally, we make use of the zero-center-mass projected diffusion to handle translation invariance similarly to~\cite{AS, liu2025asbs}.

Flow Sampling outperforms prior methods across all systems (Table~\ref{tab:synthetic}), achieving the lowest $E(\cdot)\mathcal{W}_2$ and improved concentration in low-energy regions compared to the state-of-the-art ASBS~\cite{liu2025asbs}.

\subsection{Sampling Conformers of Peptides}

We evaluate Flow Sampling on two widely-studied peptide systems: $22$-atom alanine dipeptide (Ala2) and $42$-atom alanine tetrapeptide (Ala4), which are modeled with a classical force-field and implicit water solvation using the OpenMM library \cite{eastman2017openmm}. Notably, we sample Ala2 and Ala4 conformations in \emph{full-atom} Cartesian coordinates, unlike prior sampling frameworks \cite{liu2025asbs, zhang2022path} which relied on internal torsional coordinate representations and \cite{nam2026enhancing}, which makes use of bespoke collective variables biases and pretraining data. As baseline, we take the ASBS configuration from \cite{nam2026enhancing}, but skip their MD pretraining step. Both Flow Sampling and ASBS utilize the same $E(3)$-equivariant PaiNN architecture~\cite{schutt2021equivariant}.

\noindent \textbf{Results }
For Ala2, we follow prior works~\citep{zhang2022path, liu2025asbs} and report both the 1D KL divergence of each torsional dihedral angles $(\phi,\psi)$ and the Jensen-Shannon divergence (JSD) of the 2D joint against ground truth samples generated via long molecular dynamics (MD) simulation. We also report the energy $\mathcal{W}_2$ to measure concentration of samples about low energy. Flow sampling performs favorably compared to ASBS on all metrics shown in Table~\ref{tab:ala2} and accurately recovers the ground truth density of the backbone dihedral angles shown by Ramachandran plots in Figure~\ref{fig:all_rama}. Additionally, \Cref{tab:ala2} shows that Flow Sampling appears to degrade in performance more gracefully than ASBS as the training NFE decreases, where ASBS fails at NFE 128 and 64, as indicated by the large energy-$\mathcal{W}_2$ metric. Lower model NFE can accelerate training and inference, which we discuss in greater detail later in~\Cref{subsec:efficiency}. Additional experiment details energy histogram plots can be found in~\Cref{app:ala_additional}.

For Ala4, we report a qualitative study of mode coverage and sample concentration via a 3D scatter plot over torsion angles $(\phi_1,\phi_2,\phi_3)$, which gives rise to $2^3=8$ metastable modes corresponding to the possible combinations of the three backbone torsional states. 
As depicted in~\Cref{fig:ala4_torsions}, Flow Sampling covers all 8 modes observed in the MD reference samples. 
The samples are also more tightly concentrated around the MD-supported torsional regions than those generated by ASBS, suggesting improved alignment with the relevant conformational landscape. 
This indicates that Flow Sampling can capture the multimodal structure of Ala4 while maintaining better sample concentration in physically plausible regions.
  \begin{table}                                                                                                                    
  \centering                                                                                                                       
      \captionof{table}{Results for ALA2. We report $D_{\mathrm{KL}}$ on 1-D marginals for $\phi$ and $\psi$, Jensen--Shannon      
  divergence (JSD) on the joint distribution of torsion angles $(\phi,\psi)$, and an energy-histogram $\mathcal{W}_2$ metric with  
  respect to ground-truth samples. Light and dark shading denote the best and second-best method respectively.}
      \label{tab:ala2}

      \setlength{\tabcolsep}{3pt}
      \renewcommand{\arraystretch}{1.0}

      \resizebox{\linewidth}{!}{%
      \begin{tabular}{@{}lccccc@{}}
      \toprule
      % Stacked the headers using \multirow to bridge the \cmidrule nicely
      \multirow{2}{*}{Method} & \multirow{2}{*}{\begin{tabular}[c]{@{}c@{}}NFE \\ (Train)\end{tabular}} &
  \multicolumn{2}{c}{$D_{\mathrm{KL}}$ 1-D $(\downarrow)$} & \multicolumn{1}{c}{$(\phi,\psi)$ joint} & \multicolumn{1}{c}{Energy}
  \\
      \cmidrule(lr){3-4} \cmidrule(lr){5-5} \cmidrule(lr){6-6}
       & & $\phi$ & $\psi$ & JSD $(\downarrow)$ & $\mathcal{W}_2$ $(\downarrow)$ \\
      \midrule
      % Used \multirow{4}{*}{...} to center the method name across the 4 rows
      \multirow{4}{*}{ASBS} & 1024 & 0.504{\color{gray}\tiny$\pm$0.006} & 0.726 {\color{gray}\tiny$\pm$0.300} &
  0.242{\color{gray}\tiny$\pm$0.042} & 8.650{\color{gray}\tiny$\pm$0.371} \\
                            & 256  & 0.412{\color{gray}\tiny$\pm$0.038} & 0.631{\color{gray}\tiny$\pm$0.419} &
  0.233{\color{gray}\tiny$\pm$0.090} & 7.290{\color{gray}\tiny$\pm$2.625} \\
                            & 128  & 0.663{\color{gray}\tiny$\pm$0.460} & 0.550{\color{gray}\tiny$\pm$0.267} &
  0.285{\color{gray}\tiny$\pm$0.158} & 1.0e7{\color{gray}\tiny$\pm$1.0e7} \\
                            & 64   & 1.127{\color{gray}\tiny$\pm$0.002} & 0.821{\color{gray}\tiny$\pm$0.006} &
  0.441{\color{gray}\tiny$\pm$0.002} & 6.0e7{\color{gray}\tiny$\pm$8.8e5} \\
      \midrule
      \multirow{4}{*}{\textbf{Flow Sampling}} & 1024 & \cellhii 0.031{\color{gray}\tiny$\pm$0.004} & \cellhii
  0.008{\color{gray}\tiny$\pm$0.002} & \cellhii 0.018{\color{gray}\tiny$\pm$0.001} & \cellhii 0.637{\color{gray}\tiny$\pm$0.267} \\
                                              & 256  & \cellhi 0.235{\color{gray}\tiny$\pm$0.067} & \cellhi
  0.211{\color{gray}\tiny$\pm$0.035} & \cellhi 0.118{\color{gray}\tiny$\pm$0.023} & \cellhi 3.266{\color{gray}\tiny$\pm$0.254} \\
                                              & 128  & 0.267{\color{gray}\tiny$\pm$0.023} & 0.344{\color{gray}\tiny$\pm$0.057} &
  0.152{\color{gray}\tiny$\pm$0.021} & 3.579{\color{gray}\tiny$\pm$0.311} \\
                                              & 64   & 0.316{\color{gray}\tiny$\pm$0.058} & 0.523{\color{gray}\tiny$\pm$0.006} &
  0.190{\color{gray}\tiny$\pm$0.005} & 4.113{\color{gray}\tiny$\pm$0.023} \\
      \bottomrule
      \end{tabular}%
      }
  \vspace{-10pt}
  \end{table}

\begin{figure}
\centering
\begin{subfigure}{0.32\linewidth} \centering \includegraphics[width=\linewidth]{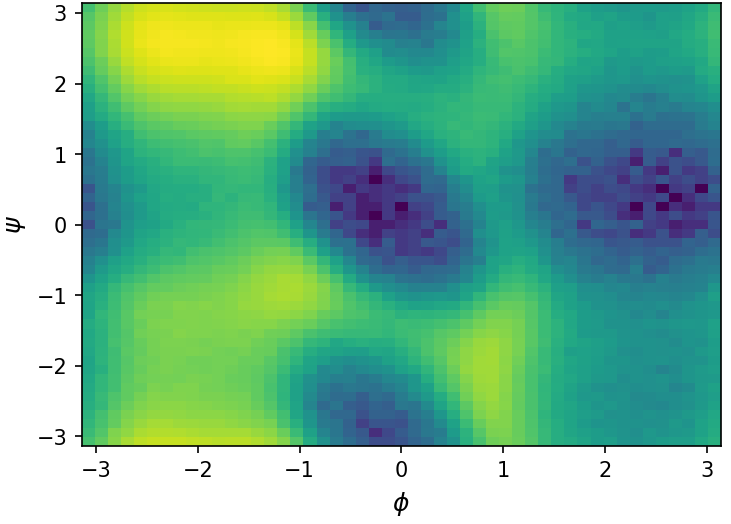} \caption*{ASBS} \label{fig:plot1} \end{subfigure} \hfill \begin{subfigure}{0.32\linewidth} \centering \includegraphics[width=\linewidth]{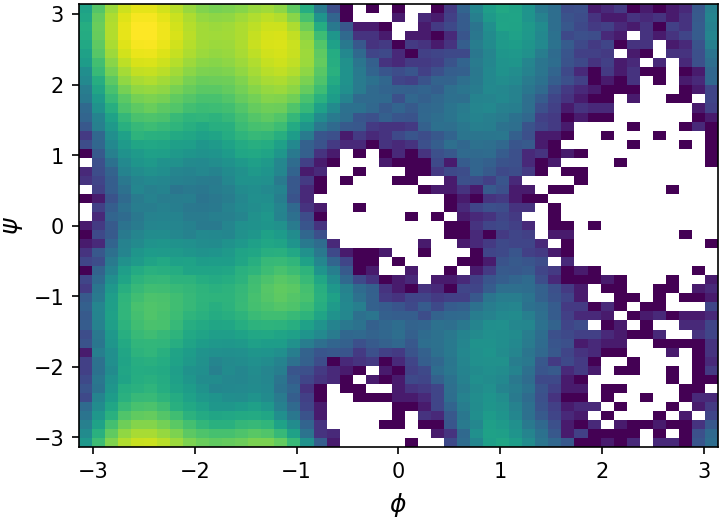} \caption*{Flow Sampling} \label{fig:plot2} \end{subfigure} \hfill \begin{subfigure}{0.32\linewidth} \centering \includegraphics[width=\linewidth]{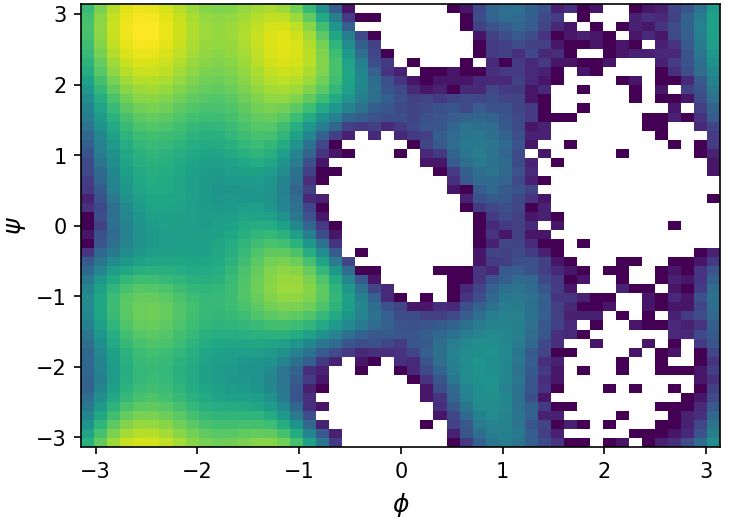} \caption*{MD} \label{fig:plot3}
\end{subfigure} \caption{Ala2 Ramachandran plots over $10^6$ samples.} \label{fig:all_rama}
\vspace{-15pt}
\end{figure}

\begin{table*}[t]
    \caption{
        Results on large-scale amortized conformer generation, evaluated on two test sets, SPICE and GEOM-DRUGS using the same 12-layer EGNN architecture across all methods. We report the coverage recall and precision (\%) and Absolute Mean RMSD (AMR) for the threshold \textbf{1.25\AA} and standard deviation across test molecules. See \Cref{sec:exp-detail} and Appendix~\ref{app:conformer_details} for details. Precision coverage was not reported in ASBS~\cite{liu2025asbs}. $^\dagger$ Re-implementation with logit normal time weighting. $^*$ Auxiliary network roughly doubles the number of network evaluation during optimization.
    \vspace{-5pt}
    }
    \label{tab:conformer_generation}
    \vskip 0.07in % for arxiv
    % \vskip 0.05in % for neurips
    \centering
    \renewcommand{\arraystretch}{1.3}
    \resizebox{\textwidth}{!}{%
\begin{tabular}{@{} l c cccc cccc @{}}
\toprule
& & \multicolumn{4}{c}{SPICE} & \multicolumn{4}{c}{GEOM-DRUGS} \\
\cmidrule(lr){3-6} \cmidrule(lr){7-10}
Method &
& \multicolumn{2}{c}{Recall} & \multicolumn{2}{c}{Precision}
& \multicolumn{2}{c}{Recall} & \multicolumn{2}{c}{Precision} \\
& NFE (Train)
& Cov. $\uparrow$ & AMR $\downarrow$ & Cov. $\uparrow$ & AMR $\downarrow$
& Cov. $\uparrow$ & AMR $\downarrow$ & Cov. $\uparrow$ & AMR $\downarrow$ \\
\midrule

RDKit ETKDG & --
& 72.74{\color{gray}\tiny$\pm$33.18} & 1.04{\color{gray}\tiny$\pm$0.52}
& \cellhii 69.68{\color{gray}\tiny$\pm$37.11} & \cellhii 1.14{\color{gray}\tiny$\pm$0.64}
& 63.51{\color{gray}\tiny$\pm$34.74} & 1.15{\color{gray}\tiny$\pm$0.61}
& \cellhii 69.77{\color{gray}\tiny$\pm$38.23} & \cellhii 1.09{\color{gray}\tiny$\pm$0.66} \\

ASBS$^*$~\cite{liu2025asbs} & 512
& 89.66{\color{gray}\tiny$\pm$19.42} & \cellhii 0.86{\color{gray}\tiny$\pm$0.24}
& -- & --
& \cellhii 74.50{\color{gray}\tiny$\pm$32.32} & \cellhii 1.05{\color{gray}\tiny$\pm$0.41}
& -- & -- \\

\multirow{1}{*}{AS$^{\dagger}$~\cite{AS}} & 256
& 88.60{\color{gray}\tiny$\pm$20.83} & \cellhi 0.87{\color{gray}\tiny$\pm$0.25}
& 59.82{\color{gray}\tiny$\pm$30.90} & 1.19{\color{gray}\tiny$\pm$0.37}
& 69.37{\color{gray}\tiny$\pm$34.44} & 1.10{\color{gray}\tiny$\pm$0.44}
& 40.15{\color{gray}\tiny$\pm$29.80} & 1.59{\color{gray}\tiny$\pm$0.54} \\
& 128
& 77.97{\color{gray}\tiny$\pm$29.9} & 0.98{\color{gray}\tiny$\pm$0.30}
&  55.30{\color{gray}\tiny$\pm$32.88} & 1.24{\color{gray}\tiny$\pm$0.40}
& 57.42{\color{gray}\tiny$\pm$35.33} &  1.24{\color{gray}\tiny$\pm$0.46}
& 33.14{\color{gray}\tiny$\pm$30.71} & 1.75{\color{gray}\tiny$\pm$0.67} \\
& 64
& 29.13{\color{gray}\tiny$\pm$--} & 1.43 {\color{gray}\tiny$\pm$--}
& 20.37{\color{gray}\tiny$\pm$--} & 1.46 {\color{gray}\tiny$\pm$--}
& 5.50 {\color{gray}\tiny$\pm$--} & 1.89 {\color{gray}\tiny$\pm$--}
& 4.15{\color{gray}\tiny$\pm$--} & 1.89 {\color{gray}\tiny$\pm$--} \\

\addlinespace[0.25em]
\midrule
\addlinespace[0.15em]

\multirow{1}{*}{\textbf{Flow Sampling}}
& 256
& \cellhii 91.89{\color{gray}\tiny$\pm$\phantom{0}7.51} & \cellhii 0.86{\color{gray}\tiny$\pm$0.23}
& \cellhi 62.35{\color{gray}\tiny$\pm$25.76} & \cellhi 1.15{\color{gray}\tiny$\pm$0.35}
& \cellhi 72.61{\color{gray}\tiny$\pm$44.81} & 1.55{\color{gray}\tiny$\pm$0.34}
& \cellhi 44.81 {\color{gray}\tiny$\pm$31.55} & \cellhi 1.55{\color{gray}\tiny$\pm$0.46} \\
& 128

&\cellhi 91.39{\color{gray}\tiny$\pm$16.63} & \cellhi 0.87{\color{gray}\tiny$\pm$0.22}
& 61.27{\color{gray}\tiny$\pm$30.55} & 1.18{\color{gray}\tiny$\pm$0.33}
& 71.20{\color{gray}\tiny$\pm$34.21} & \cellhi 1.08{\color{gray}\tiny$\pm$0.44}
& 41.63{\color{gray}\tiny$\pm$31.80} & 1.61{\color{gray}\tiny$\pm$0.62} \\
& 64
& 90.13{\color{gray}\tiny$\pm$\phantom{0}17.85} & \cellhi 0.87{\color{gray}\tiny$\pm$0.23}
& 60.75{\color{gray}\tiny$\pm$31.02} & 1.18{\color{gray}\tiny$\pm$0.33}
& 71.14{\color{gray}\tiny$\pm$34.40} & 1.09{\color{gray}\tiny$\pm$0.43}
& 40.82{\color{gray}\tiny$\pm$31.67} & 1.61{\color{gray}\tiny$\pm$0.61} \\

\bottomrule
\end{tabular}
}
\vspace{-10pt}
\end{table*}

\subsection{Amortized Conformer Generation}\label{sec:exp-detail}
We evaluate Flow Sampling on the large-scale amortized molecular conformer generation benchmark of \citet{AS}, where the target is a conditional Boltzmann distribution $r(x|g) = -\tfrac{1}{\tau} E(x | g)$, with energy given by the learned eSEN force field~\cite{fu2025learning} and conditioning $g$ that provides the molecule topology (atom types and bonds).

For a molecule with $N$ atoms, a conformer is represented as atomic coordinates $x \in \mathbb{R}^{N \times 3}$, corresponding to $\texttt{shape} = (N,3)$ in Algorithm~1. Generation is conditioned on a fixed molecular graph $g$, as input to the drift network and energy. The model therefore defines a conditional diffusion process over atomic positions.

We train on 24,775 molecular topologies from SPICE \citep{eastman2023spice} and evaluate on 80 held-out SPICE molecules and a generalization dataset of 80 GEOM-DRUGS molecules \citep{axelrod2022geom}, following the protocol of \citet{AS}. RDKit~\cite{landrum2006rdkit} is a widely used fast conformer generation tool that provides a strong baseline. For more details see~\cite{AS}.

\subsection{Von Mises-Fisher Mixtures on $\mathbb{S}^2$}
Directional and rotational data arise in applications such as  robotics and graphics, and are naturally represented on spherical domains~\cite{bullo2005geometric, bronstein2021geometric}. To validate the Riemannian Flow Sampling framework developed in \Cref{sec:hyper_sphere}, we consider sampling on the two-dimensional unit sphere
\begin{align*}
\mathbb{S}^2 = \{ x \in \mathbb{R}^3 : \|x\|_2 = 1 \}.
\end{align*}
We construct a multimodal target distribution given by a mixture of von Mises--Fisher (vMF)~\cite{banerjee2005clustering} components, the spherical analogue of Gaussians, with density
\begin{align*}
f_{\mathrm{vMF}}(x) \propto \exp(\kappa\, \mu^\top x)
\end{align*}
parameterized by mean directions $\mu \in \mathbb{S}^2$ and concentration $\kappa \ge 0$.
We then construct a mixture distribution
\begin{equation}
q(x) \propto \sum_{k=1}^{K} w_k\, f_{\mathrm{vMF}}(x; \mu_k, \kappa_k),
\end{equation}
where the mixture components are evenly weighted ($w_k = 1/K$) and the mean directions $\{\mu_k\}_{k=1}^K$ are chosen to lie along the coordinate axes and diagonal directions of $\mathbb{R}^3$.

Using the closed-form conditional drift derived in \Cref{sec:hyper_sphere}, we apply iterative Flow Sampling directly on the manifold in ambient space without reparameterizations. As shown in Figure~\ref{fig:vMF_vis}, the learned diffusion accurately recovers the multimodal target density, demonstrating that Flow Sampling extends naturally beyond Euclidean spaces and can operate directly on curved geometries.

\subsection{Discussion on model efficiency}\label{subsec:efficiency}
Training computational cost per a single exploration, followed by a training phase~\Cref{alg:fs_train} is composed of three primary components: i) Number of energy gradient evaluations ii) Number of function evaluations (NFE (Train)) to sample $X_1^{\bar{\theta}}\sim p_1^{\bar{\theta}}$, and iii) number of optimization steps. Regarding i) i.e. number of energy gradient evaluations, AS~\cite{AS}, ASBS~\cite{liu2025asbs} and Flow Sampling are equally efficient. Regarding iii), AS and Flow Sampling are performing 100 gradient updates per optimization stage. However, since ASBS requires fitting an auxiliary corrector network which doubles the cost of optimization and results in about $\times 2$ gradient updates per optimization stage. Most importantly, ii) the NFE to acquire a sample from the model using the Euler-Maruyama solver. 

Tables~\ref{tab:ala2} and~\ref{tab:conformer_generation} show that Flow Sampling maintains consistent performance across metrics even when substantially reducing NFE (Train). 
In particular, reducing NFE from 512 to 64 on conformer generation in~\Cref{tab:conformer_generation} potentially cuts the exploration-phase cost by $\times 8$ relative to the 512-NFE baseline, which dominates the training cost. 
Furthermore, we observe strong degradation in the low-NFE regime for the baselines: ASBS on Ala2 in~\Cref{tab:ala2} and AS on conformer generation in~\Cref{tab:conformer_generation}. 
In contrast, Flow Sampling remains robust in both settings, indicating that the efficiency of the exploration phase significantly improves the scalability of the method for larger models or datasets.

\section{Conclusion}
We introduce Flow Sampling, a principled method for learning to sample from unnormalized densities by matching conditional denoising diffusion drifts~(\ref{eq:cond_drift_path}). We present a simplified formulation for training diffusion-based samplers, and provide a pytorch pseudo-code in Algorithm~\ref{alg:fs_train}.  We extend diffusion samplers to non-Euclidean geometries, where we provide closed-form formulas for the conditional drift on constant curvature Riemannian manifolds. We validate our Flow Sampling method empirically, observing strong results in the synthetic-energy dataset as shown in Table~\ref{tab:synthetic} and on small peptides in~\Cref{tab:ala2}. On the amortized conformer generation benchmark reported in Table~\ref{tab:conformer_generation}, our Flow Sampling demonstrates a $4\times$--$8\times$ speedup in simulation training, potentially enabling a considerable increase in scale of diffusion-based samplers.

\textbf{Limitations} \ \
Flow Sampling is trained as a fixed-point procedure: since target samples
are unavailable, the ideal conditional diffusion-matching objective is
approximated using samples from the detached current sampler. This is
effective empirically, but we do not yet provide global convergence
guarantees for the resulting replay-buffer dynamics. Finally,
our Riemannian extension relies on closed-form conditional drifts for
constant-curvature manifolds; extending the same efficiency to more general
manifolds and constrained spaces remains an important direction.

\newpage

\subsection*{Impact Statement}
This work is primarily intended for scientific and industrial research applications and does not introduce immediate risks of misuse. Nevertheless, as with advances in generative modeling more broadly, care should be taken to ensure responsible deployment, especially in application domains such as drug discovery and materials science, where modeling biases or inaccuracies could influence downstream decisions.

Flow Sampling has direct relevance to computational chemistry, including molecular modeling and conformer generation, and may contribute to accelerating research in drug development, materials engineering, and biophysical modeling.
\bibliography{example_paper}
\bibliographystyle{icml2026}

%%%%%%%%%%%%%%%%%%%%%%%%%%%%%%%%%%%%%%%%%%%%%%%%%%%%%%%%%%%%%%%%%%%%%%%%%%%%%%%
%%%%%%%%%%%%%%%%%%%%%%%%%%%%%%%%%%%%%%%%%%%%%%%%%%%%%%%%%%%%%%%%%%%%%%%%%%%%%%%
% APPENDIX
%%%%%%%%%%%%%%%%%%%%%%%%%%%%%%%%%%%%%%%%%%%%%%%%%%%%%%%%%%%%%%%%%%%%%%%%%%%%%%%
%%%%%%%%%%%%%%%%%%%%%%%%%%%%%%%%%%%%%%%%%%%%%%%%%%%%%%%%%%%%%%%%%%%%%%%%%%%%%%%
\newpage
\appendix
\onecolumn
\section{Synthetic Energy Experiment Details}\label{app:synthetic_energy}

\subsection{Double Well Potential (DW-4)}
We use the same double-well potential as in iDEM~\citep{akhound2024iterated}, which was originally proposed in~\citet{kohler2020equivariant}. DW-4 describes a pair-wise distance potential energy for a system of 4 particles $\{x_1, x_2, x_3, x_4 \}$, where each particle has 2 spatial dimensions $x_i \in \mathbb{R}^2$ $(d=8)$. The potentials analytical form is given by:
\begin{equation}
    E(x) = \frac{1}{\tau}\sum_{ij} a (d_{ij} - d_0) + b (d_{ij} - d_0)^2 + c (d_{ij} - d_0)^4,\quad d_{ij} = \lVert x_i - x_j \rVert_2
\end{equation}
where we set $a=0$, $b=-4$, $c=0.9$ and temperature $\tau = 1$.
\subsection{Lennard-Jones Potential (LJ-13, LJ-55)}
Similarly to DW-4, the Lennard-Jones potential is also based pair-wise distances of a system with $n$ particles, but each each particle has 3 spatial dimensions. Its analytical form is given by
\begin{equation}
    E^{\text{LJ}}(x) = \frac{\epsilon}{\tau} \sum_{ij} \left( \left(\frac{r_m}{d_{ij}}\right)^6 - \left(\frac{r_m}{d_{ij}}\right)^{12} \right),\quad  d_{ij} = \lVert x_i - x_j \rVert_2
\end{equation}
where \( r_m, \tau, \epsilon \) and \( c \) are physical constants. As in \citet{kohler2020equivariant} and \citet{akhound2024iterated}, we add the additional a harmonic potential:

\begin{equation}
    E^{\text{osc}}(x) = \frac{1}{2} \sum_i \| x_i - x_{\text{COM}} \|^2
\end{equation}

where \( x_{\text{COM}} \) refers to the center of mass of the system. Therefore, the final energy is then \( E^{\text{Tot}} = E^{\text{LJ}}(x) + c E^{\text{osc}}(x) \), for \( c \) the oscillator scale. As in previous work, we use \( r_m = 1, \tau = 1, \epsilon = 1 \) and \( c = 1.0 \).

\subsection{Architectures and Hyperparameters}
\label{app:synthetic_hparams}

We provide the main model and training hyperparameters used for the synthetic energy experiments in~\Cref{tab:synthetic_hparams}. 
For all three systems, we use an EGNN architecture following the setup of~\citet{akhound2024iterated}. 
Results for iDEM are taken from the AS evaluation in~\cite{AS}. 
Each outer iteration consists of an exploration phase, where new samples are generated and their energy gradients are evaluated, followed by an optimization phase over the replay buffer.

\begin{table}
\centering
\caption{
Hyperparameters for the synthetic energy experiments. 
We report the architecture, replay-buffer configuration, and optimization settings used for DW-4, LJ-13, and LJ-55.
}
\label{tab:synthetic_hparams}
\small
\setlength{\tabcolsep}{7pt}
\begin{tabular}{lccc}
\toprule
\textbf{Hyperparameter} & \textbf{DW-4} & \textbf{LJ-13} & \textbf{LJ-55} \\
\midrule
\multicolumn{4}{l}{\textit{Model}} \\
Hidden features & 128 & 128 & 128 \\
EGNN layers & 3 & 5 & 5 \\
\midrule
\multicolumn{4}{l}{\textit{Base SDE}} \\
Harmonic source $\sigma$ & $1$ & $1$ & $1$ \\
Noise scaling $\gamma$ & Adaptive, Eq.~\eqref{eq:gamma_adapt} ($c=1$) & Adaptive, Eq.~\eqref{eq:gamma_adapt} ($c=1$) & Adaptive, Eq.~\eqref{eq:gamma_adapt} ($c=1$) \\
\midrule
\multicolumn{4}{l}{\textit{Replay buffer}} \\
Buffer size & 10{,}000 & 10{,}000 & 10{,}000 \\
New samples per outer iteration & 1{,}024 & 1{,}024 & 128 \\
Batch size & 512 & 512 & 128 \\
\midrule
\multicolumn{4}{l}{\textit{Training}} \\
Epochs & 5{,}000 & 5{,}000 & 5{,}000 \\
Iterations per Epoch & 200 & 300 & 300 \\
Learning rate & $3\times 10^{-4}$ & $3\times 10^{-4}$ & $3\times 10^{-4}$ \\
Reward clipping value & 100.0 & 100.0 & 100.0 \\
\bottomrule
\end{tabular}
\end{table}

\paragraph{Adaptive noise scaling.}
Since the magnitude of the reward-gradient signal can vary substantially across systems and during training, we adapt the diffusion coefficient $\gamma$ using the current replay buffer. 
Let $\mathcal{B}$ denote the replay-buffer distribution over stored samples $x_1$ and their reward gradients. 
We set
\begin{equation}\label{eq:gamma_adapt}
\gamma =
\frac{c}{
\sqrt{
\mathbb{E}_{x_1 \sim \mathcal{B}}
\left[
\left\lVert \nabla r(x_1) \right\rVert^2
\right]
+ \varepsilon
}}
,
\end{equation}
where $c$ controls the scale of the reward-gradient contribution and $\varepsilon>0$ is a small constant for numerical stability. 
In practice, the expectation is estimated from the stored gradients in the replay buffer. 
This makes $\gamma$ inversely proportional to the empirical RMS reward-gradient norm, reducing sensitivity to the absolute gradient scale and to the manual clipping threshold.
\subsection{Reported Metrics}\label{app:w2_metric}

\paragraph{Geometric $\mathcal{W}_2$}

Because our energy and diffusion process is invariant to rotational and permutations symmetries, we also take these symmetries into account when measuring distance between generated and ground truth point clouds (\eg, from long run MCMC). Note this $\mathcal{W}_2$ metric is different from what is reported in~\citet{akhound2024iterated}, which uses the euclidean metric.

\begin{equation}
    \mathcal{W}_2(\mu, \nu) = \left( \inf_{\pi} \int \pi(x,y) d(x,y)^2 \,\diff{x}\diff{y} \right)^{\frac{1}{2}}
\end{equation}

where \( \pi \) is the transport plan with marginals constrained to \( \mu \) and \( \nu \) respectively. Here the distance metric $d$ takes into account all point-cloud symmetries, and is obtained by minimizing the squared Euclidean distance over all possible combinations of rotations, reflections ($O(d)$), and permutations ($\mathbb{S}(k)$) for $k$ particles of spatial dimension $d$.

\begin{equation}
    d(x_0, x_1) = \min_{R \in O(d), P \in \mathbb{S}(k)} \| x_0 - (R\otimes P) x_1 \|_2^2.
\end{equation}

However, computing the exact minimal squared distance is computationally infeasible in practice. Therefore, we adopt the approach of~\citet{kohler2020equivariant} and approximate the minimizer by performing a sequential search

\begin{equation}
    d(x_0, x_1) \approx \min_{R \in SO(d)} \| x_0 - (R\otimes \tilde{P}) x_1 \|_2^2, \quad
    \tilde{P} = \arg\min_{P \in \mathbb{S}(k)} \| x_0 - P x_1 \|_2^2.
\end{equation}

\paragraph{Energy $\mathcal{W}_2$ ($E(\cdot)\,\mathcal{W}_2$)}
An informative way to assess the quality of the generated samples is to look at their energy distribution with respect to the ground truth distribution obtained from energies of long-run MCMC simulations. This shows how well the generated samples are avoiding high-energy regions. Specifically, we use the Wasserstein-2 ($\mathcal{W}_2$) distance on the 1-dimensional energy distribution on $\mathbb{R}$. This results in the standard euclidean $\mathcal{W}_2$ metric

\begin{equation}
\mathcal{W}_2(E_{\mu}, E_{\nu}) = \left( \inf_{\pi} \int \pi(x,y)|x - y|^2 \diff{x}\diff{y} \right)^{\frac{1}{2}},
\end{equation}

\section{Alanine Dipeptide and Tetrapeptide Experiment Details}\label{app:ala_additional}

\paragraph{Energy Function}
For both Ala2 and Ala4, we utilize a classical force-field and implicit water solvation using the OpenMM library \cite{eastman2017openmm}, with settings identical to~\cite{nam2026enhancing}.

\paragraph{Model Architecture}
Following the setup from~\cite{nam2026enhancing}, we employ the PaiNN architecture~\cite{schutt2021equivariant} which is an $E(3)$--equivariant graph neural network specifically suited for molecular energy and force-field prediction. \cite{nam2026enhancing} modifies the architecture to include time-conditioning,  breaking parity symmetries, which restricts the symmetry to spatial $SE(3)$-equivariance. We use the same hyper-parameters which are restated below in~\Cref{tab:ala_hparams}.

\paragraph{Reference Data}
For Ala2, we use the same $10^7$ test samples from~\cite{midgley2023flow}, also used by~\cite{nam2026enhancing}. These samples are generated via replica exchange MD across 21 replicas of various temperatures spanning $300$K to $1300$K and require $2.3 \times 10^{10}$ energy evaluations to generate.
  
For Ala4, we directly use the test set of $5\times 10^7$ samples generated by~\cite{nam2026enhancing}, but only subsampling $50{,}000$ for the generated 3D scatter plot~\Cref{fig:ala4_torsions}. These samples required approximately $10^{11}$ energy evaluations to generate.

\paragraph{Ramachandran Jensen--Shannon Divergence (JSD).}                                                                                                    
  We evaluate the quality of generated ala2 conformations via the Jensen--Shannon divergence (JSD) of the joint backbone dihedral                                                                  
  $(\phi,\psi)$ distribution.  For both the generated and reference sample                                                         
  sets, we compute the Ramachandran angles $\phi$ and $\psi$ using
  \texttt{mdtraj} and bin them into a $200\times200$ 2D histogram over
  $[-\pi,\pi]\times[-\pi,\pi]$.  Each histogram is normalized to form a
  discrete probability distribution:
  \begin{equation}
      \hat{p}_{ij} = \frac{h_{ij}}{\sum_{i,j} h_{ij}},
  \end{equation}
  where $h_{ij}$ denotes the count in bin $(i,j)$.  The JSD between the
  generated distribution $\hat{p}$ and the reference distribution $\hat{q}$
  is then
  \begin{equation}
      \mathrm{JSD}(\hat{p}\,\|\,\hat{q})
      = \frac{1}{2}\,D_{\mathrm{KL}}(\hat{p}\,\|\,m)
      + \frac{1}{2}\,D_{\mathrm{KL}}(\hat{q}\,\|\,m),
      \qquad
      m = \tfrac{1}{2}(\hat{p}+\hat{q}),
  \end{equation}
  where $D_{\mathrm{KL}}(p\,\|\,q) = \sum_{i,j} p_{ij}\log\frac{p_{ij}}{q_{ij}}$ is the Kullback--Leibler divergence.
  
\subsection{Interatomic Distance and Energy Histograms}
In \cref{fig:ala2_hist}, we provide additional qualitative comparisons on Ala2 by comparing the energy and interatomic distance distributions of generated samples against MD reference data.
\begin{figure}
    \centering
    \includegraphics[width=0.8\linewidth]{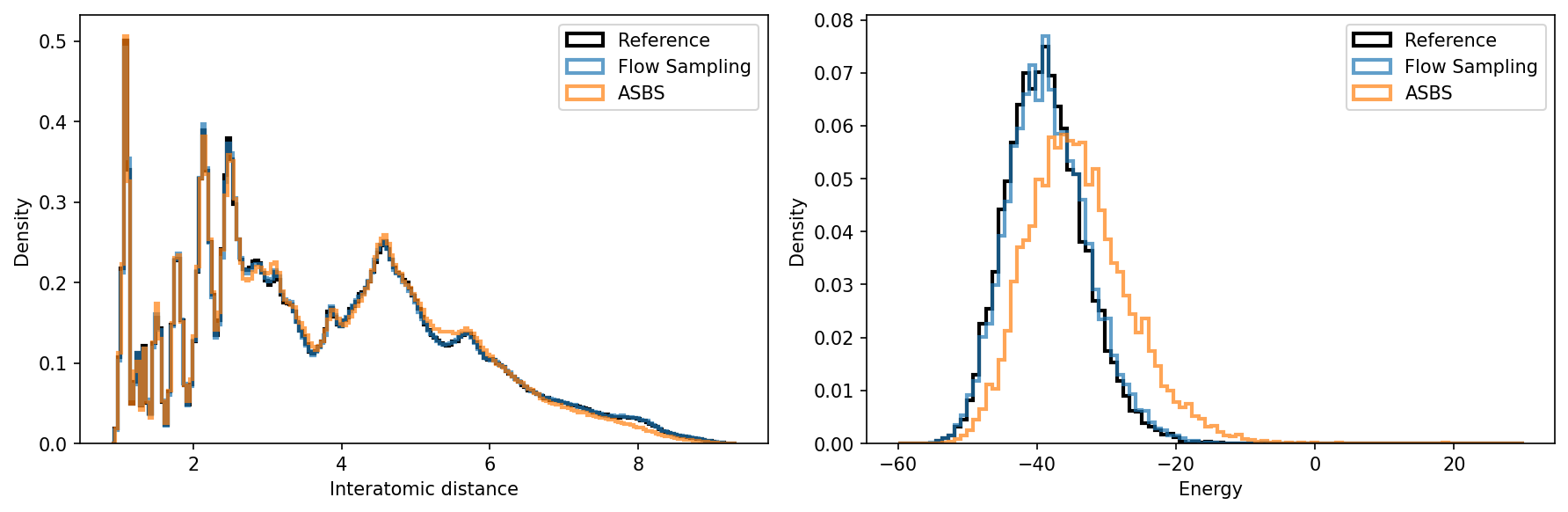}
    \caption{Ala2 interatomic distance and energy histogram over $10^4$ samples evaluated against the MD data as reference.}
    \label{fig:ala2_hist}
\end{figure}
\section{Conformer Generation Experiment Details}\label{app:conformer_details}
\subsection{Coverage Recall and Precision Metrics}\label{app:rmsd_metrics}
\begin{table}
\centering
\caption{
Hyperparameters for the alanine peptide experiments. Model and training
hyperparameters used for alanine dipeptide (Ala2) and alanine tetrapeptide
(Ala4).
}
\label{tab:ala_hparams}
\small
\setlength{\tabcolsep}{10pt}
\begin{tabular}{lcc}
\toprule
\textbf{Hyperparameter} & \textbf{Ala2} & \textbf{Ala4} \\
\midrule
\multicolumn{3}{l}{\textit{Model}} \\
Hidden dimension & 128 & 128 \\
Radial basis & 64 & 64 \\
Control layers & 4 & 4 \\
Corrector layers & 3 & 3 \\
Radial cutoff [\AA] & 8.0 & 8.0 \\
\midrule
\multicolumn{3}{l}{\textit{Base SDE}} \\
Harmonic Source $\sigma$ & $2 \text{\AA}$ & $2 \text{\AA}$\\
Noise scaling $\gamma$ & Adaptive, Eq.~\eqref{eq:gamma_adapt} ($c=1$) & Adaptive, Eq.~\eqref{eq:gamma_adapt} ($c=1$) \\
\midrule
\multicolumn{3}{l}{\textit{Replay buffer}} \\
Buffer size & 16{,}384 &  16{,}384 \\
Buffer samples per epoch & 2{,}048 & 2{,}048 \\
Batch size & 2{,}048 & 2{,}048 \\
\midrule
\multicolumn{3}{l}{\textit{Training}} \\
Epochs & 10{,}000 & 10{,}000\\
Iterations per Epoch & 100 & 100\\
Initial learning rate & \(1\times 10^{-4}\) & \(1\times 10^{-4}\) \\
Final learning rate & \(3\times 10^{-5}\) & \(5\times 10^{-5}\) \\
Reward clipping value & 1000.0 & 100.0 \\
\bottomrule
\end{tabular}
\end{table}
\textbf{Root Mean Square Deviation (RMSD)} Similarly to \citet{ganea2021geomol} and \citet{jing2022torsional}, we measure the so-called Average Minimum RMSD (AMR) and Coverage (COV) for Precision (P) and Recall (R). When measuring these metrics, we  generate twice as many conformers as provided by CREST reference conformers. For $K = 2L$ let $\{C^*_l\}_{l \in [1,L]}$ and $\{C_k\}_{k \in [1,K]}$ be the sets of ground truth and generated conformers respectively. In our evaluations, we use $L = \text{max}(L', 128)$, where $L'$ is the number of reference conformers given by CREST, taking the lowest energy conformers as a subset. The RMSD metric finds the best average distance between atoms of molecule with respect to the reference molecule, taking into account all possible symmetries.

\textbf{Coverage Recall}
\begin{align}
    \text{COV-R}(\delta) &:= \frac{1}{L} \left| \left\{ l \in \{1,\ldots,L\} : \exists k \in \{1,\ldots, K\}, \quad \text{RMSD}(C_k, C^*_l) < \delta \right\} \right| \\
    \text{AMR-R} &:= \frac{1}{L} \sum_{l \in \{1,\ldots, L\}} \min_{k \in \{1,\ldots, K\}} \text{RMSD}(C_k, C^*_l)
\end{align}
\textbf{Coverage Precision}
\begin{align}
    \text{COV-P}(\delta) &:= \frac{1}{K} \left| \left\{ k \in \{1,\ldots,K\} : \exists l \in \{1,\ldots, L\}, \quad \text{RMSD}(C_k, C^*_l) < \delta \right\} \right| \\
    \text{AMR-P} &:= \frac{1}{K} \sum_{k \in \{1,\ldots, K\}} \min_{l \in \{1,\ldots, L\}} \text{RMSD}(C_k, C^*_l)
\end{align}
where $\delta > 0$ is the coverage threshold.

\subsection{Hyperparameters and Architecture Details for SPICE and GEOM-DRUGS}

We use the existing hyperparameter settings for AS and ASBS. Below we specify the hyperparameters specific to Flow Sampling, noting that all EGNN architectures are kept the same for all methods.

\textbf{Flow Sampling: }
We use an Equivariant Graph Neural Network (EGNN, \citealt{satorras2021n}) with 12 layers and a hidden feature dimension of 128. The model is trained for $5000$ outer-loop iterations, sampling $100$ batches each iteration from the replay buffer. Each GPU maintains its own buffer with max size of 64000 samples and at each iteration we generate $128 \times 8$ new molecules and energy evaluations for the buffer across 8 GPUs. The model is trained with a batch size of 64 per GPU and uses a constant gamma schedule $\gamma_t \equiv 0.1$. Weight gradient clipping of $10^{20}$ is applied. Additionally we use the temperature $\tau=5\times 10^{-3}$ and a regularization constant $\alpha=100.0$ for the regularized energy function. The regression target, the temperature-scaled gradient of the energy function is clipped at $\ell^2$ norm 150.

\section{Connection to Adjoint Sampling via Control Variates}\label{app:as_connection}

In this section we relate the Adjoint Sampling~\cite{AS} target to the target obtained
from Flow Sampling when the supervising path is defined using the same
Brownian bridge pair law. The comparison is made at the target endpoint
law. In particular, we take \(X_1\sim q\), construct the Brownian bridge
interpolant used by Adjoint Sampling, and apply Flow Sampling to the
conditional path induced by this interpolant. The resulting Flow Sampling
target differs from the Adjoint Sampling target by a samplewise control
variate whose conditional expectation given \(X_t\) is zero.

Adjoint Sampling considers the controlled SDE
\begin{align}
    \diff X_t
    =
    \sigma(t)u(X_t,t)\diff t+\sigma(t)\diff B_t,
    \qquad
    X_0=0 .
    \label{eq:as_controlled_sde}
\end{align}
The base process is obtained by setting \(u\equiv 0\), with marginals
\begin{align}
    p_t^{\mathrm{base}}
    =
    \mathcal N(0,\nu_t I),
    \qquad
    \nu_t
    :=
    \int_0^t \sigma(s)^2\diff s .
\end{align}
For a target \(q(x)\propto e^{r(x)}\), the terminal cost used by Adjoint Sampling, corresponding the optimal control formulation, is
\begin{align}
    g(x)
    =
    \log p_1^{\mathrm{base}}(x)-r(x).
    \label{eq:as_terminal_cost}
\end{align}
In this zero-base-drift sampling setting, the
Adjoint Sampling regression target for the control is
\begin{align}
    -\sigma(t)\nabla g(X_1).
    \label{eq:as_control_target}
\end{align}

\textbf{Supervising Process} \ \
We now define the Flow Sampling supervising path using the same Brownian
bridge pair law. Given \(X_1\sim q\), let
\begin{align}
    X_t
    =
    \lambda_t X_1+s_t\varepsilon,
    \qquad
    \lambda_t:=\frac{\nu_t}{\nu_1},
    \qquad
    s_t^2:=\frac{\nu_t(\nu_1-\nu_t)}{\nu_1},
    \qquad
    \varepsilon\sim\mathcal N(0,I).
    \label{eq:as_bridge_interpolant}
\end{align}
Conditioned on \(\varepsilon\), the map \(X_1\mapsto X_t\) is
deterministic and invertible for \(t\in(0,1)\), and therefore pushes
\(q\) forward to the conditional path
\begin{align}
    p_{t|\varepsilon}(x|\varepsilon)
    =
    \frac{1}{\lambda_t^{d}}
    q\!\left(
        \frac{x-s_t\varepsilon}{\lambda_t}
    \right).
    \label{eq:as_bridge_conditional_path}
\end{align}
Applying Flow Sampling to this conditional path gives a control target
\(u^{\mathrm{FS}}_{t|\varepsilon}\). The result below shows that this
target has the same conditional regression signal as the Adjoint Sampling
target in~\eqref{eq:as_control_target}.

\begin{theorem}[Flow Sampling recovers the Adjoint Sampling target]
Let \(X_t=\lambda_t X_1+s_t\varepsilon\) be the Brownian bridge
interpolant in~\eqref{eq:as_bridge_interpolant}, with \(X_1\sim q\), and
let \(u^{\mathrm{FS}}_{t|\varepsilon}(X_t|\varepsilon)\) be the Flow
Sampling control target associated with the conditional path
\eqref{eq:as_bridge_conditional_path}. Then, for every \(t\in(0,1)\),
\begin{align}
    \E\left[
        u^{\mathrm{FS}}_{t|\varepsilon}(X_t|\varepsilon)
        \,\middle|\,
        X_t=x
    \right]
    =
    \E\left[
        -\sigma(t)\nabla g(X_1)
        \,\middle|\,
        X_t=x
    \right].
    \label{eq:fs_recovers_as_target}
\end{align}
Consequently, the two samplewise targets define the same least-squares
regression target as functions of \(X_t\).
\end{theorem}

\begin{proof}
We first compute the Flow Sampling target for the conditional path
\eqref{eq:as_bridge_conditional_path}. Along the interpolant,
\begin{align}
    v_{t|\varepsilon}(X_t|\varepsilon)
    =
    \dot\lambda_tX_1+\dot s_t\varepsilon, \qquad
    \nabla\log p_{t|\varepsilon}(X_t|\varepsilon)
    =
    \frac{1}{\lambda_t}\nabla r(X_1).
\end{align}
Using
\begin{align}
    \dot\lambda_t
    =
    \frac{\sigma(t)^2}{\nu_1},
    \qquad
    \dot s_t
    =
    \frac{\sigma(t)^2(1-2\lambda_t)}{2s_t},
\end{align}
the Flow Sampling physical drift target is
\begin{align}
    b^{\mathrm{FS}}_{t|\varepsilon}(X_t|\varepsilon)
    &=
    v_{t|\varepsilon}(X_t|\varepsilon)
    +
    \frac{\sigma(t)^2}{2}
    \nabla\log p_{t|\varepsilon}(X_t|\varepsilon) \\
    &=
    \sigma(t)^2
    \left[
        \frac{X_1}{\nu_1}
        +
        \frac{1-2\lambda_t}{2s_t}\varepsilon
        +
        \frac{1}{2\lambda_t}\nabla r(X_1)
    \right].
\end{align}
Since the controlled SDE~\eqref{eq:as_controlled_sde} is parameterized by
the physical drift \(\sigma(t)u(X_t,t)\), we compare targets at the level
of the control:
\begin{align}
    u^{\mathrm{FS}}_{t|\varepsilon}(X_t|\varepsilon)
    &:=
    \frac{1}{\sigma(t)}
    b^{\mathrm{FS}}_{t|\varepsilon}(X_t|\varepsilon) \\
    &=
    \sigma(t)
    \left[
        \frac{X_1}{\nu_1}
        +
        \frac{1-2\lambda_t}{2s_t}\varepsilon
        +
        \frac{1}{2\lambda_t}\nabla r(X_1)
    \right].
    \label{eq:fs_bridge_control_target}
\end{align}

We now compare this with the Adjoint Sampling target. Since
\(p_1^{\mathrm{base}}=\mathcal N(0,\nu_1 I)\), the terminal cost
\eqref{eq:as_terminal_cost} satisfies
\begin{align}
    -\sigma(t)\nabla g(X_1)
    &=
    -\sigma(t)
    \nabla
    \left(
        \log p_1^{\mathrm{base}}(X_1)-r(X_1)
    \right) \\
    &=
    \sigma(t)
    \left[
        \frac{X_1}{\nu_1}
        +
        \nabla r(X_1)
    \right].
    \label{eq:as_target_expanded}
\end{align}
Subtracting~\eqref{eq:as_target_expanded} from
\eqref{eq:fs_bridge_control_target} gives
\begin{align}
    u^{\mathrm{FS}}_{t|\varepsilon}(X_t|\varepsilon)
    -
    \bigl(-\sigma(t)\nabla g(X_1)\bigr)
    &=
    \frac{\sigma(t)(1-2\lambda_t)}{2}
    \left[
        \frac{\varepsilon}{s_t}
        +
        \frac{1}{\lambda_t}\nabla r(X_1)
    \right].
    \label{eq:fs_as_control_variate}
\end{align}
Thus the two targets differ pointwise by the term on the right-hand side.
We now show that this term vanishes after conditioning on \(X_t\).

Fix \(X_t=x\). The conditional density of $X_1\sim q$ given $X_t=x$ under the Brownian bridge law is given by,
\begin{align}
    p(X_1=x_1\mid X_t=x)
    &\propto
    q(x_1)
    \exp\!\left(
        -\frac{\norm{x-\lambda_t x_1}^2}{2s_t^2}
    \right).
    \label{eq:posterior_x1_given_xt_bridge}
\end{align}
Moreover, under this conditional law,
\begin{align}
    \varepsilon
    =
    \frac{x-\lambda_tX_1}{s_t}.
\end{align}
Applying Stein's identity~\cite{hyvarinen2005estimation} to
\eqref{eq:posterior_x1_given_xt_bridge} gives
\begin{align}
    0
    &=
    \E\left[
        \nabla_{X_1}\log p(X_1\mid X_t=x)
        \,\middle|\,
        X_t=x
    \right] \\
    &=
    \E\left[
        \nabla r(X_1)
        +
        \frac{\lambda_t}{s_t^2}
        (x-\lambda_tX_1)
        \,\middle|\,
        X_t=x
    \right] \\
    &=
    \E\left[
        \nabla r(X_1)
        +
        \frac{\lambda_t}{s_t}\varepsilon
        \,\middle|\,
        X_t=x
    \right].
\end{align}
Dividing by \(\lambda_t\), we obtain
\begin{align}
    \E\left[
        \frac{\varepsilon}{s_t}
        +
        \frac{1}{\lambda_t}\nabla r(X_1)
        \,\middle|\,
        X_t=x
    \right]
    =
    0.
    \label{eq:fs_as_zero_mean_cv}
\end{align}
Combining~\eqref{eq:fs_as_control_variate} with
\eqref{eq:fs_as_zero_mean_cv} proves
\eqref{eq:fs_recovers_as_target}.

Finally, least-squares regression onto functions of \(X_t\) depends only
on the conditional expectation of the samplewise target. Therefore
\eqref{eq:fs_recovers_as_target} implies that the Flow Sampling target
and the Adjoint Sampling target define the same regression target as
functions of \(X_t\).
\end{proof}

This result is a target-law comparison. In the practical fixed-point
training procedures used by both methods, the endpoint samples are obtained
from the current detached sampler and are reused through a replay buffer.
The calculation above explains why, when the endpoint law has reached the
target distribution, the Brownian bridge Flow Sampling target reduces to
the same conditional regression signal as Adjoint Sampling, with the
remaining sample-wise difference acting as a zero-mean control variate.

\section{Proofs}
\label{a:proofs}
\FlowToDiffusion*
\begin{proof}
      $u_t$~(\ref{eq:v_to_u}) generate a diffusion process with marginal $p_t$, if they satisfy the Fokker–Planck equation~(\ref{eq:fokker_planck}). Indeed,
    \begin{align*}
        \frac{\partial}{\partial t}&p_{t}(x) + \nabla\brac{p_{t}(x) u_{t}(x)}\\
        &=\underbrace{\frac{\partial}{\partial t}p_{t}(x) + \nabla\brac{p_t(x)v_{t}(x)}}_{=0 \text{, by continuity equation.} }+ \frac{g_t^2}{2}\nabla\brac{p_t(x)\nabla\log p_t(x)}\\
    % &= \frac{\partial}{\partial t}p_{t}(x) + \nabla\brac{p_t(x)v_{t}(x)} + \frac{g_t^2}{2}\nabla\brac{p_t(x)\nabla\log p_t(x)}\\
    &= \frac{g_t^2}{2}\nabla^2p_{t}(x),
\end{align*}
where in the first equality we substitute $u_t$~, and in the second equality we used the assumption that $v_t$ generate $p_t$, \ie, they satisfy the continuity equation~(\ref{eq:continuity}).
\end{proof}

\Drift*
\begin{proof}[Proof of Proposition~\ref{prop:drift}.]
    By definition, the drift~(\ref{eq:cond_drift}) $u_{t|0}$ at $X_t$ the interpolant~(\ref{eq:x_cond_fs}) is
    \begin{equation}
        u_{t|0}\parr{X_t|x_0} = v_{t|0}\parr{X_t|x_0} + \frac{g_t^2}{2}\nabla\log p_{t|0}\parr{X_t}.
    \end{equation}
    We treat each term separately. The first term was also shown by \citet{lipman2023flow}, since the $X_t$ the interpolant~(\ref{eq:x_cond_fs}) is the solution to the flow defined by $v_{t|0}$ the conditional velocity~(\ref{eq:v_cond_fs}), we have
    \begin{align}
        v_{t|0}\parr{X_t|x_0} &= \frac{d}{dt}X_t\\
        &= \frac{d}{dt}\parr{\alpha_tX_1 +\sigma_tx_0}\\
        &= \dot{\alpha}_tX_1  + \dot{\sigma}_tx_0,
    \end{align}
    where the first equality is by definition~(\ref{e:flow_process}), and the second equality we substitute $X_t$ as in equation~(\ref{eq:x_cond_fs}).

    Second, using the definition of the conditional probability path~(\ref{eq:p_cond_fs}) $p_{t|0}$, we have
    \begin{align}
        \nabla\log p_{t|0}\parr{X_t|x_0} &= \nabla\parr{\log q\parr{\frac{x-\sigma_tx_0}{\alpha_t}} -d\log\parr{\alpha_t}}\Big\vert_{x=X_t}\\
        &= \nabla\parr{r\parr{\frac{x-\sigma_tx_0}{\alpha_t}} -\log\parr{Z} -d\log\parr{\alpha_t}}\Big\vert_{x=X_t}\\
        &= \frac{1}{\alpha_t}\nabla r\parr{\frac{X_t -\sigma_tx_0}{\alpha_t}}\\
        &= \frac{1}{\alpha_t}\nabla r\parr{X_1},
    \end{align}
    where in the second equality we substitute the definition for $q(x)=\frac{e^{r(x)}}{Z}$, and in the last equality we substitute $X_t$ as in equation~(\ref{eq:x_cond_fs}).
\end{proof}

\JClosedFormProp*
\begin{proof}[Proof of Proposition~\ref{prop:J_closed_form}]
    Our proof closely follows results about Jacobi fields from the book Introduction to Riemannian Manifolds~\citep{lee2018introduction}.

    Denote the geodesic starting at $X_1$ and ending at $x_0$ by
    \begin{equation}
        \gamma(s)\in\gM,\forall s\in[0,1]\quad \gamma(0) = X_1,\quad    \gamma(1) = x_0.
    \end{equation}
    The velocity of the geodesic is defined as $\gamma'(s)=\frac{d}{ds}\gamma(s)$, and for every $s\in[0,1]$, $\gamma'(s)\in T_{\gamma(s)}\gM$.
    Indeed
    \begin{align}
        \inner{\gamma'(s)}{\gamma(s)}_{\Sigma} & = \inner{\frac{d}{ds}\gamma(s)}{\gamma(s)}_{\Sigma}\\
        &= \frac{1}{2}\parr{\inner{\frac{d}{ds}\gamma(s)}{\gamma(s)}_{\Sigma} + \inner{\gamma(s)}{\frac{d}{ds}\gamma(s)}_{\Sigma}}\\
        & = \frac{1}{2}\frac{d}{ds}\inner{\gamma(s)}{\gamma(s)}_{\Sigma} = 0.
    \end{align}
    The covariant derivative $D_s$ on our embedded manifold can be understood as $D_sv = P^{\perp}_{\gamma(s)}\frac{d}{ds}v(s)$ for any field $v(s)\in T_{\gamma(s)}\gM$ along the geodesic $\gamma$. A field $v(s)\in T_{\gamma(s)}\gM$ is said to be parallel along $\gamma(s)$ if $D_sv\equiv 0$. By definition of the geodesic, the velocity $\gamma'(s)$ is parallel along $\gamma(s)$, \ie,
    \begin{equation}
        D_s\gamma'\equiv 0
    \end{equation}
    Every parallel field $v(s)$ along $\gamma(s)$ given by an isometric linear map $T_{\gamma(0)\too\gamma(s)}:T_{\gamma(0)}\gM\too T_{\gamma(s)}\gM$ called the \emph{parallel transport} along $\gamma$ and an initial condition,
    \begin{equation}
        v(s) = T_{\gamma(0)\too\gamma(s)}v(0).
    \end{equation}
    For our set of manifolds~(\ref{eq:manifold}), for any $X_1\in\gM$ and $x_0\in\gM$ with a unique minimizing geodesic,
    \begin{equation}
        \gamma(s) = \exp_{X_1}\parr{sV_1}, \quad V_1=\log_{X_1}\parr{x_0},
    \end{equation}
    where $\exp_{X_1}:T_{X_1}\too\gM$ and $\log_{X_1}:\gM\too T_{X_1}$ are the exponential and logarithmic map of the manifold. By definition of the exponential map
    \begin{equation}
        \gamma'(0) = V_1
    \end{equation}
    Note, both the basis of $\exp_{X_1}$ and $\log_{X_1}$ depends on $X_1$, hence the Jacobian of geodesic requires the total variation w.r.t. $X_1$.
    \begin{equation}
    D_{X_1}\gamma(s) = \underbrace{ (d_x \exp_{X_1})_{sV_1}}_{\text{Base-point variation}} + \underbrace{ (d_v \exp_{X_1})_{sV_1}}_{\text{Argument variation}} \circ \parr{ s D_{X_1}\log_{X_1}(x_0)}
    \end{equation}

     To compute the variations of the exponent map, we use tools of Jacobi fields. A field $j(s)$ along $\gamma(s)$ is said to be a Jacobi field if it satisfies the Jacobi equation. In constant curvature manifolds, the Jacobi equation takes the simple form
    \begin{equation}\label{eq:jacobi_equation}
        D^2_sj(s) +\kappa\brac{\inner{\gamma'(s)}{\gamma'(s)}_{\Sigma}j(s) - \inner{\gamma'(s)}{j(s)}_{\Sigma}\gamma'(s)} = 0,
    \end{equation}
    Proposition 10.1~\citep{lee2018introduction} states that every variation through the geodesic is a Jacobi field. That is, for every vector $v_1\in T_{X_1}\gM$, both
    \begin{equation}
        j_x(s) = (d_x \exp_{X_1})_{sV_1}\cdot v_1,\quad\text{and}\quad j_v(s)=(d_v \exp_{X_1})_{sV_1}\cdot v_1
    \end{equation} are Jacobi fields, where "$\cdot$" denotes the action of the variation on $v_1$.

    Additionally, Proposition 10.2~\citep{lee2018introduction} states that for every $s_0\in[0,1]$ and pair of vectors $v,w\in T_{\gamma(s_0)}\gM$ there exists a unique Jacobi field $j(s)$ along gamma such that
    \begin{equation}
        j(s_0) = v \quad D_sj(s_0) = w.
    \end{equation}
    We use the uniqueness to compute the variations $(d_x \exp_{X_1})_{sV_1}$ and $(d_v \exp_{X_1})_{sV_1}$. That is, we have that at $s=0$, for every $v_1\in T_{X_1}\gM$,
    \begin{equation}\label{eq:jacobi_boundary_conditions}
        (d_x \exp_{X_1})_{0}\cdot v_1 = v_1,\quad D_s(d_x \exp_{X_1})_{0}\cdot v_1 = 0,\quad\text{and}\quad (d_v \exp_{X_1})_{0}\cdot v_1 = 0,\quad D_s(d_v \exp_{X_1})_{0}\cdot v_1 = v_1.
    \end{equation}
    Hence by uniqueness it is suffice to find two Jacobi fields $j_x(s)$ and $j_v(s)$ along $\gamma(s)$ that satisfy these boundary conditions.

    %%% tangential and normal fields
    A field $v^{\parallel}(s)\in T_{\gamma(s)}\gM$ is called \emph{tangential field} if $v^{\parallel}(s)\in\text{span}\set{\gamma'(s)}$ for every $s\in[0,1]$, and a field $v^{\perp}(s)\in T_{\gamma(s)}\gM$ is called \emph{normal field} if $\inner{v(s)}{\gamma'(s)}_{\Sigma}=0$ for every $s\in[0,1]$. We show that the Jacobian can be decomposed to a tangential field and a normal field. Importantly, the Jacobi equation~(\ref{eq:jacobi_equation}) takes a simpler form in case $j^{\parallel}(s)$ is a tangential field or $j^{\perp}$ is a normal field. Indeed
    \begin{equation}
        D^2_sj^{\parallel}(s) = 0,\quad\text{and}\quad D_s^2j^{\perp}(s) + \kappa \inner{\gamma'(s)}{\gamma'(s)}_{\Sigma}j^{\perp}(s)=0.
    \end{equation}
    Additionally, since $\gamma'(s)$ is parallel along $\gamma$,
    \begin{align}
        \inner{\gamma'(s)}{\gamma'(s)}_{\Sigma} &= \inner{T_{\gamma(0)\too\gamma(s)}\gamma'(0)}{T_{\gamma(0)\too\gamma(s)}\gamma'(0)}_{\Sigma}\\
        &= \inner{\gamma'(0)}{\gamma'(0)}_{\Sigma}\\
        &= \inner{V_1}{V_1}_{\Sigma}\\
        &= \omega_1^2,
    \end{align}
    where in the second equality used that the parallel transport $T_{\gamma(0)\too\gamma(s)}$ is an isometry. Thus the Jacobi equation~(\ref{eq:jacobi_equation}) further simplify to
    \begin{equation}\label{eq:jacobi_equation_decomposed}
        D^2_sj^{\parallel}(s) = 0,\quad\text{and}\quad D_s^2j^{\perp}(s) + \kappa w_1^2j^{\perp}(s)=0.
    \end{equation}
    The simplest way to generate fields that conserve orientation along $\gamma(s)$, \ie, tangential or normal are using the parallel transport. Since $T_{X_1\too\gamma(s)}$ is an isometry if $v_1^{\parallel}, v_1^{\perp} \in T_{X_1}\gM$ are tangential and normal vectors (reps.), \ie, $v_1^{\parallel}\in\text{span}\set{V_1}$ and $\inner{v_1^{\perp}}{V_1}_{\Sigma}=0$, then $v_1^{\parallel}$ and $v_1^{\perp}$ give rise to a tangential and normal (resp.) parallel field along $\gamma(s)$ defined as
    \begin{equation}
        v^{\parallel}(s) = T_{X_1\too\gamma(s)}v_1^\parallel,\quad\text{and}\quad v^{\perp}(s) = T_{X_1\too\gamma(s)}v_1^\perp.
    \end{equation}
    Indeed,
    \begin{equation}
        v^{\parallel}(s) = T_{X_1\too\gamma(s)}v_1^\parallel \in\text{span}\set{T_{X_1\too\gamma(s)}V_1} = \text{span}\set{\gamma'(s)},
    \end{equation}
    and
    \begin{equation}
        \inner{v^{\perp}(s)}{\gamma'(s)}_{\Sigma} = \inner{T_{X_1\too\gamma(s)}v_1^\perp}{T_{X_1\too\gamma(s)}V_1}_{\Sigma} = \inner{v_1^\perp}{V_1}_{\Sigma} =0.
    \end{equation}

    First we consider the tangential direction. Let $v_1^{\parallel}\in\text{span}\set{V_1}$, and consider $j^{\parallel}_x(s)=v^{\parallel}(s)$ and $j^{\parallel}_v(s)=sv^{\parallel}(s)$, both tangential fields and since $v^{\parallel}(s)$ is parallel $D_sv^{\parallel}\equiv 0$. Thus both $j^{\parallel}_x(s)$ and $j^{\parallel}_v(s)$  satisfy the Jacobi equation for the tangential case~\ref{eq:jacobi_equation_decomposed}, and the initial conditions at $s=0$ by construction are
    \begin{equation}
        j^{\parallel}_x(0) = v_1^{\parallel},\quad D_sj^{\parallel}_x(0) = 0,\quad\text{and}\quad j^{\parallel}_v(0) = 0,\quad D_sj^{\parallel}_v(0) = v_1^{\parallel}.
    \end{equation}
    Hence uniqueness implies,
    \begin{equation}
        j^{\parallel}_x(s) = (d_x \exp_{X_1})_{sV_1}\cdot v_1^{\parallel},\quad\text{and}\quad j^{\parallel}_v(s)=(d_v \exp_{X_1})_{sV_1}\cdot v_1^{\parallel}.
    \end{equation}

    Second we consider the normal directions. Let $v_1^{\perp}\in T_{X_1}\gM$ be such that $\inner{v_1^{\perp}}{V_1}_{\Sigma}=0$. and consider $j^{\perp}(s) = c_s(X_1,x_0)v^{\perp}(s)$, where $c(X_1,x_0):[0,1]\too\R$ is some smooth scaling function. Then since $v^{\perp}(s)$ is parallel field along $\gamma(s)$, \ie, $D_sv^{\perp}\equiv0$, the normal case of the Jacobi equation~(\ref{eq:jacobi_equation_decomposed}) reduce to
    \begin{equation}
        c''_s(X_1,x_0)v^{\perp}(s) + \kappa w_1^2 c_s(X_1,x_0)v^{\perp}(s) = 0,
    \end{equation}
    which is true if and only if $c_s\parr{X_1, x_0}$ satisfies the ODE,
    \begin{equation}\label{eq:jacobi_perp_reduced}
        c''_s(X_1,x_0) + \kappa w_1^2c_s(X_1,x_0) = 0.
    \end{equation}
    The solutions to the ODE in equation~\eqref{eq:jacobi_perp_reduced} are
    \begin{equation}
        c_s(X_1,x_0) = \begin{cases}
            c_0(X_1,x_0)\cos\parr{\sqrt{\kappa}\omega_1s} + \frac{c'_0(X_1,x_0)}{\sqrt{\kappa}\omega_1}\sin\parr{\sqrt{\kappa}\omega_1s} & \textbf{ if } \kappa>0\\
            c_0(X_1,x_0)\cosh\parr{\sqrt{|\kappa|}\omega_1s} + \frac{c'_0(X_1,x_0)}{\sqrt{\kappa}\omega_1}\sinh\parr{\sqrt{|\kappa|}\omega_1s} & \textbf{ if } \kappa<0
        \end{cases},
    \end{equation}
    where $c_0(X_1,x_0)$ and $c'_0(X_1,x_0)$ are determined by the boundary conditions at $s=0$. Thus we define,
    \begin{equation}
        j_x^{\perp}(s) = c_s^{x}(X_1,x_0)v^{\perp}(s),\quad\text{and}\quad j_v^{\perp}(s) = c_s^{v}(X_1,x_0)v^{\perp},
    \end{equation}
    where
    \begin{equation}
         c_s^{x}(X_1,x_0) = \begin{cases}
            \cos\parr{\sqrt{\kappa}\omega_1s} & \textbf{ if } \kappa>0\\
            \cosh\parr{\sqrt{|\kappa|}\omega_1s}  & \textbf{ if } \kappa<0
        \end{cases},
    \end{equation}
    and
    \begin{equation}
        c^{v}_s(X_1,x_0) = \begin{cases}
             \frac{1}{\sqrt{\kappa}\omega_1}\sin\parr{\sqrt{\kappa}\omega_1s} & \textbf{ if } \kappa>0\\
             \frac{1}{\sqrt{\kappa}\omega_1}\sinh\parr{\sqrt{|\kappa|}\omega_1s} & \textbf{ if } \kappa<0.
        \end{cases}
    \end{equation}
    Then they satisfy the initial conditions
    \begin{equation}
        j^{\perp}_x(0) = v_1^\perp,\quad D_sj^{\perp}_x(0) = 0,\quad\text{and}\quad j^{\perp}_v(0) = 0,\quad D_sj^{\perp}_v(0) = v_1^\perp.
    \end{equation}
    Hence uniqueness implies,
    \begin{equation}
        j^{\perp}_x(s) = (d_x \exp_{X_1})_{sV_1}\cdot v_1^{\perp},\quad\text{and}\quad j^{\perp}_v(s)=(d_v \exp_{X_1})_{sV_1}\cdot v_1^{\perp}.
    \end{equation}
    Let $v_1\in T_{X_1}\gM$, finally we can write the full Jacobi fields for the variations of the exponent, \ie,
    \begin{equation}
        j_x(s) = (d_x \exp_{X_1})_{sV_1}\cdot v_1,\quad\text{and}\quad j_v(s)=(d_v \exp_{X_1})_{sV_1}\cdot v_1.
    \end{equation}
    Define $v_1^\parallel=P_{V_1}v_1$ and $v_1^\perp=P_{V_1}^\perp v_1$, where $P_{V_1}$ is the orthogonal projection~(\ref{eq:orthogonal_projection}) and $P_{V_1}^{\perp}$ its orthogonal compliment. Then $v_1^\parallel$ is tangential vector and $v_1^\perp$ is normal vector, thus they give rise to Jacobi fields $j_x^{\parallel}$, $j_v^{\parallel}$ which satisfy the tangential Jacobi equation~(\ref{eq:jacobi_equation_decomposed}), and $j_x^{\perp}$, $j_v^{\perp}$ which satisfy the normal Jacobi equation~(\ref{eq:jacobi_equation_decomposed}). Lastly, since the Jacobi equation~(\ref{eq:jacobi_equation}) decomposes to two independent linear differential equations (\ie, one tangential and one normal) as in equation~(\ref{eq:jacobi_equation_decomposed}), the fields
    \begin{align}
        j_x(s) &= j_x^{\parallel}(s) + j_x^{\perp}\\
        &= T_{X_1\too \gamma(s)}P_{V_1}v_1 + c_s^x\parr{X_1,x_0}T_{X_1\too \gamma(s)}P_{V_1}^{\perp}v_1\\
        &= \brac{T_{X_1\too \gamma(s)}P_{V_1} + c_s^x\parr{X_1,x_0}T_{X_1\too \gamma(s)}P_{V_1}^{\perp}}v_1
    \end{align}
    and
    \begin{align}
        j_v(s) &= j_v^{\parallel}(s) + j_v^{\perp}\\
        &= sT_{X_1\too \gamma(s)}P_{V_1}v_1 + c_s^v\parr{X_1,x_0}T_{X_1\too \gamma(s)}P_{V_1}^{\perp}v_1\\
        &= \brac{sT_{X_1\too \gamma(s)}P_{V_1} + c_s^v\parr{X_1,x_0}T_{X_1\too \gamma(s)}P_{V_1}^{\perp}}v_1,
    \end{align}
    are Jacobi fields and they satisfy boundary condition~(\ref{eq:jacobi_boundary_conditions}).

    As last step, we compute the variation of the logarithmic $D_{X_1}\log_{X_1}\parr{x_0}$. By definition
    \begin{equation}
        x_0 = \gamma(1) = \exp_{X_1}\parr{\log_{X_1}\parr{x_0}}.
    \end{equation}
    Since $x_0$ is fixed,
    \begin{align}
        0  &= D_{X_1}\gamma(1) = (d_x \exp_{X_1})_{V_1} + (d_v \exp_{X_1})_{V_1} \circ D_{X_1}\log_{X_1}(x_0)\\
        &= \brac{T_{X_1\too x_0}P_{V_1} + c_1^x\parr{X_1,x_0}T_{X_1\too x_0}P_{V_1}^{\perp}} + \brac{T_{X_1\too x_0}P_{V_1} + c_1^v\parr{X_1,x_0}T_{X_1\too x_0}P_{V_1}^{\perp}}\circ  D_{X_1}\log_{X_1}(x_0)\\
        &= T_{X_1\too x_0}\circ\brac{\parr{P_{V_1} + c_1^x\parr{X_1,x_0}P_{V_1}^{\perp}} + \parr{P_{V_1} + c_1^v\parr{X_1,x_0}P_{V_1}^{\perp}}\circ D_{X_1}\log_{X_1}(x_0)}\\
        &= T_{X_1\too x_0}\circ\brac{\parr{P_{V_1} + P_{V_1}\circ D_{X_1}\log_{X_1}(x_0)} + \parr{c_1^v\parr{X_1,x_0}P_{V_1}^{\perp} + c_1^v\parr{X_1,x_0}P_{V_1}^{\perp}\circ D_{X_1}\log_{X_1}(x_0)}}.
    \end{align}
    Since the parallel transport is invertible, this implies
    \begin{equation}
        P_{V_1}\circ D_{X_1}\log_{X_1}(x_0) = - P_{V_1},\quad\text{and}\quad P_{V_1}^{\perp}\circ D_{X_1}\log_{X_1}(x_0) = - \frac{c_1^x\parr{X_1,x_0}}{c_1^v\parr{X_1,x_0}}P_{V_1}^{\perp}
    \end{equation}

    We conclude that the Jacobian $D_{X_1}\gamma(s)$ is
    \begin{align}
        D_{X_1}\gamma(s) &= (d_x \exp_{X_1})_{sV_1} + (d_v \exp_{X_1})_{sV_1} \circ D_{X_1}\log_{X_1}(x_0)\\
        &= \brac{T_{X_1\too x_0}P_{V_1} + c_s^x\parr{X_1,x_0}T_{X_1\too\gamma(s)}P_{V_1}^{\perp}} + \brac{sT_{X_1\too\gamma(s)}P_{V_1} + c_s^v\parr{X_1,x_0}T_{X_1\too\gamma(s)}P_{V_1}^{\perp}}\circ D_{X_1}\log_{X_1}(x_0)\\
        &=T_{X_1\too\gamma(s)}\brac{P_{V_1} +sP_{V_1}\circ D_{X_1}\log_{X_1}(x_0)} + T_{X_1\too\gamma(s)}\brac{c_s^x\parr{X_1,x_0}P_{V_1}^{\perp} + c_s^v\parr{X_1,x_0}P_{V_1}^{\perp}\circ D_{X_1}\log_{X_1}(x_0)}\\
        &=\parr{1-s}T_{X_1\too\gamma(s)}P_{V_1} +\frac{c_1^v\parr{X_1,x_0}c_s^x\parr{X_1,x_0} - c_1^x\parr{X_1,x_0}c_s^v\parr{X_1,x_0}}{c_1^v\parr{X_1,x_0}}T_{X_1\too\gamma(s)}P_{V_1}^{\perp}\\
        &=\parr{1-s}T_{X_1\too\gamma(s)}P_{V_1} +c_{1-s}\parr{X_1,x_0}T_{X_1\too\gamma(s)}P_{V_1}^{\perp},
    \end{align}
    where
    \begin{equation}
        c_{1-s}(X_1, x_0) = \begin{cases}
        \dfrac{\sin((1-s)\omega_1\sqrt{\kappa})}{\sin(\omega_1\sqrt{\kappa})} & \text{if } \kappa > 0 \\[2ex]
        \dfrac{\sinh((1-s)\omega_1\sqrt{|\kappa|})}{\sinh(\omega_1\sqrt{|\kappa|})} & \text{if } \kappa < 0
    \end{cases},
    \end{equation}
    Finally, noting that the geodesic interpolant~(\ref{eq:x_cond_fs_riem}) is $X_t=\gamma(1-t)$, which implies $V_1=-\dot{X}_1$, and therefore $\text{span}\set{V_1}=\text{span}\set{\dot{X}_1}$, completes the proof.
    \end{proof}

\end{document}